\documentclass{article}
\usepackage{subfiles} %To add appendix
\PassOptionsToPackage{numbers,compress}{natbib}
     \usepackage[preprint]{neurips_2020}

\usepackage[utf8]{inputenc} % allow utf-8 input
\usepackage[T1]{fontenc}    % use 8-bit T1 fonts
\usepackage{hyperref}       % hyperlinks
\usepackage{url}            % simple URL typesetting
\usepackage{booktabs}       % professional-quality tables
\usepackage{amsfonts}       % blackboard math symbols
\usepackage{nicefrac}       % compact symbols for 1/2, etc.
\usepackage{microtype}      % microtypography

\usepackage{amsmath,graphicx}
\usepackage{amsthm}
\usepackage{algorithm,algorithmic}

\usepackage{xcolor}
\usepackage{caption}
\usepackage{subcaption}
\usepackage{multirow,array}
\usepackage{rotating}

\DeclareMathOperator*{\argmax}{arg\,max}

\newcommand{\Ex}{\mathbb{E}\hspace{0.05cm}}
\newcommand{\bm}[1]{\mbox{\boldmath $#1$}}
\newcommand{\bs}{\boldsymbol}
\newcommand{\ba}{\left[\begin{array}}
	\newcommand{\ea}{\\\end{array} \right]}

\newcommand{\define}{\stackrel{\Delta}{=}}

\newtheorem{definition}{Definition}
\newtheorem{assumption}{Assumption}

\newtheorem{remark}{Remark}
\newtheorem{theorem}{Theorem}
\newtheorem{lemma}{Lemma}

\title{ISL: A novel approach for deep exploration}

\author{
	Lucas C. Cassano\thanks{The author is also with EPFL.} \\
	Department of Electrical and Computer Engineering\\
	University of California\\
	Los Angeles, USA\\
	\texttt{cassanolucas@ucla.edu} \\
	\And
	Ali H. Sayed \\
	School of Engineering\\
	\'Ecole Polytechnique F\'ed\'erale de Lausanne \\
	CH-1015 Lausanne, Switzerland\\
	\texttt{ali.sayed@epfl.ch}
}

\begin{document}

\maketitle

\begin{abstract}
  In this article we explore an alternative approach to address \textit{deep exploration} and we introduce the ISL algorithm, which is efficient at performing \textit{deep exploration}. Similarly to maximum entropy RL, we derive the algorithm by augmenting the traditional RL objective with a novel regularization term. A distinctive feature of our approach is that, as opposed to other works that tackle the problem of \textit{deep exploration}, in our derivation both the learning equations and the exploration-exploitation strategy are derived in tandem as the solution to a well-posed optimization problem whose minimization leads to the optimal value function. Empirically we show that our method exhibits state of the art performance on a range of challenging deep-exploration benchmarks.
\end{abstract}

\section{Introduction}
\label{introduction}

Reinforcement learning is concerned with designing algorithms that seek to maximize long term cumulative rewards by interacting with an environment of unknown dynamics. Recently, significant progress has been made in deep RL (the combination of RL algorithms with Deep Learning). In particular, several algorithms have been introduced based on the maximum entropy framework  \cite{G_learning,haarnoja2017reinforcement,pcl,sac,trust_pcl,sbeed,haarnoja2018learning}. In maximum entropy RL, the traditional RL objective is augmented with the entropy of the learned policy, which is weighted by a temperature parameter. The main benefit of the maximum RL framework is that it allows to derive algorithms that are more sample efficient (because they operate off-policy) than policy gradients algorithms like A3C \cite{a3c}, TRPO \cite{TRPO} and PPO \cite{ppo} and exhibit improved stability over algorithms based on Q-learning. Furthermore, it has been demonstrated empirically that algorithms based on the maximum entropy framework have improved exploration capabilities \cite{haarnoja2017reinforcement}.

Although these recent algorithms have shown state of the art performance in several tasks, they are not without shortcomings. In the first place, augmenting the original RL objective with the entropy bonuses biases the solution of the optimization problem. This could be solved by slowly annealing the temperature parameter, but it is unclear what is the best way to do so and how this change affects convergence speed. The second, and more important, shortcoming is that although these algorithms have improved exploration capabilities, they are still unable to perform deep exploration (i.e., they tend to perform poorly in environments with very sparse reward structures). These difficulties suggest that there may not be sufficient theoretical justification for the use of the policy's entropy as the regularizer for the RL objective, which in turn raises the question of whether another regularizer could help address the two aforementioned challenges. In this work we explore this possibility and connect the problem of deep exploration with that of maximum entropy RL. Traditionally, the problem of deriving deep exploration strategies has been treated separately from that of deriving the learning equations \cite{stadie2015incentivizing,boot_dqn,ube,burda2018large,osband2018randomized}. In this article we show that both the learning equations and the exploration strategy can be derived in tandem as the solution to a well-posed optimization problem whose minimization leads to the optimal value function. We do this by augmenting the objective function with a novel regularizer. Instead of using the policy's entropy as the regularizer (as maximum entropy RL does), we use a function which depends on the uncertainties over the agent's estimates of the $q$ values.

The contribution of this paper is the introduction of an alternative approach to address \textit{deep exploration} and the introduction of an algorithm that we refer to as the \textit{Information Seeking Learner} (ISL). This algorithm has a similar form to soft RL algorithms like SAC and SBEED, but the fundamental difference is that it explicitly estimates the uncertainties of its $q$ estimates and uses these uncertainties to drive deep exploration.

\subsection{Relation to prior work}

ISL is related to recent work on maximum entropy algorithms \cite{G_learning,haarnoja2017reinforcement,pcl,sac,trust_pcl,sbeed}. All these algorithms augment the traditional RL objective with a term that aims to maximize the entropy of the learned policy, which is weighted by a temperature parameter. The consequence of using this augmented objective is two-fold. First, it allows to derive off-policy algorithms with improved stability compared to algorithms based on Q-learning. Second, it improves the exploration properties of the algorithms. However, using this augmented objective has two main drawbacks. First, the policy to which these algorithms converge is biased away from the true optimal policy. This point can be handled by annealing the temperature parameter but this can slow down convergence and compromise exploration. Furthermore, it is unclear what the optimal schedule to perform annealing is and how it affects the conditioning of the optimization problem. Secondly, although exploration is improved, algorithms derived from this modified cost are not efficient at performing deep exploration. The reason for this is that a unique temperature parameter is used for all states. In order to perform deep exploration it is necessary to have a scheme that allows agents to learn policies that exploit more in states where the agent has high confidence in the optimal action (in order to be able to reach further novel states with high probability) and act in a more exploratory manner in unknown states. In \cite{cassano2019team} the authors explore this possibility in a tabular setting making the temperature parameter state dependent and annealing it progressively as states are explored; however, the annealing schedule is determined heuristically and it is unclear what would be an effective way to do so when using function approximation. The relation between ISL and these works is that we also augment the RL objective with an entropy based regularizer and as a consequence we derive a an algorithm that relies on a soft backup. Under our scheme, agents converge to the true optimal policy without the need for annealing any parameter and, moreover, an exploration strategy is derived that is capable of performing deep exploration.

Our paper is also related to works on deep exploration strategies. One line of research is based on posterior sampling for RL \cite{boot_dqn,ube,osband2018randomized}. The main idea in these works is to sample the value function from a posterior distribution over such function. The common point between these works and ours is that deep exploration is driven by some measure of the uncertainty on the estimated value function. One difference between these works and ours is that they rely on Posterior Sampling while our work does not. Another fundamental difference is that these works treat the exploration problem separately from the derivation of the learning rules. In other words, they use heuristics based on randomised value functions to drive exploration but use off-the-shelf learning rules like Q-learning. In our approach, we use the uncertainty over the learned value function not only to drive exploration but to formulate a new objective from which we derive a new learning rule. Another line of work relies on pseudo-counts \cite{bellemare2016unifying,ostrovski2017count}, the fundamental limitation of this approach is that it only performs good if the generalization of the density model is aligned with the task objective \cite{osband2018randomized}. Another line of research relies on IRs \cite{stadie2015incentivizing,vime,pathak2017curiosity,burda2018exploration}, these rewards are based on some heuristic based proxy that reflects the novelty of a state. ISL can also be seen as utilizing IRs, however there are two main differences with the previous works. Firstly, we use IRs that depend on the uncertainty of the $q$ function. Secondly, we never actually compute IRs, but rather we rely on a different backup to update the $q$ estimates (this point is clarified in section \ref{section:derivation}).

\section{Preliminaries}
\label{preliminaries}
We consider the problem of policy optimization within the traditional RL framework. We model our setting as a Markov Decision Process (MDP), with an MDP defined by ($\mathcal{S}$,$\mathcal{A}$,$\mathcal{P}$,$r$), where $\mathcal{S}$ is a set of states, $\mathcal{A}$ is a set of actions of size $A=|\mathcal{A}|$, $\mathcal{P}(s'|s,a)$ specifies the probability of transitioning to state $s'\in\mathcal{S}$ from state $s\in\mathcal{S}$ having taken action $a\in\mathcal{A}$, and $r:\mathcal{S}\times\mathcal{A}\times\mathcal{S}\rightarrow\mathbb{R}$ is the reward mapping $r(s,a,s')$.
\begin{assumption}\label{assumtpion:r_lims}
	$r_{\min}\leq r(s,a,s')\leq r_{\max},$ $\forall(s,a,s')$.
\end{assumption}
In this work we consider the maximization of the discounted infinite reward as the objective of the RL agent:
\begin{equation}\label{eq:opt_policy}
\pi^\dagger(a|s)=\argmax_{\pi}\Ex_{\mathcal{P},\pi}\bigg(\sum_{t=0}^{\infty}\gamma^{t}\bm{r}(\bm{s}_t,\bm{a}_t,\bm{s}_{t\hspace{-0.2mm}+\hspace{-0.2mm}1})\big|\bm{s}_0\hspace{-1mm}=\hspace{-1mm}s\bigg)
\end{equation}
where $\pi^\dagger(a|s)$ is the optimal policy, $\gamma\in[0,1)$ is the discount factor and $\bm{s}_t$ and $\bm{a}_t$ are the state and action at time $t$, respectively. We clarify that in this work, random variables are always denoted in bold font. We recall that each policy $\pi$ has an associated state value function $v^\pi(s)$ and state-action value function $q^\pi(s,a)$,\footnote{In this paper we will refer to both $v^\pi(s)$ and $q^\pi(s,a)$ as value functions indistinctly.} and that value functions corresponding to the optimal policy are given by \cite{puterman}:
\begin{align}\label{eq:optimal_value_function_definition}
	&q^{\dagger}(s,a)=r(s,a)+\gamma\Ex_{\mathcal{P}}\max_{a'}q^\dagger(\bm{s}',a'),\hspace{5mm}\pi^\dagger(a|s)=\argmax_{\pi}\sum_{a}\pi(a|s)q^{\dagger}(s,a)
\end{align}
where for convenience we defined $r(s,a)=\Ex\bm{r}(s,a,\bm{s}')$.

\section{Algorithm derivation}
\label{section:derivation}
Equations \eqref{eq:optimal_value_function_definition} are useful to derive algorithms for planning problems (i.e., problems in which the reward function and transition kernel are known) but may be unfit to derive RL algorithms because they mask the fact that the agent relies on estimated quantities (which are subject to uncertainty). Hence, in this work, we modify \eqref{eq:opt_policy} to reflect the fact that an RL agent is constrained by the uncertainty of its estimates. Intuitively, we change the goal of the agent to not just maximize the discounted cumulative rewards but also to collect information about the MDP in order to minimize the uncertainty of its estimated quantities. For this purpose, we assume that at any point in time the agent has some estimate of the optimal value function denoted by $\widehat{q}(s,a)$, which is subject to some error $\widetilde{q}(s,a)=q^\dagger(s,a)-\widehat{q}(s,a)$. We model the unknown quantities $\widetilde{q}(s,a)$ as random variables. More specifically, for each state-action pair we assume $\bs{\widetilde{q}}(s,a)$ follows a uniform probability distribution with zero mean $\bs{\widetilde{q}}(s,a)\sim\textrm{U}(0,\ell(s,a))$ such that:
\begin{align}\label{eq:l_condition}
\widetilde{q}(s,a)\in[\widehat{q}(s,a)-\ell(s,a);\widehat{q}(s,a)+\ell(s,a)]
\end{align}
We will refer to the probability density function of $\bs{\widetilde{q}}(s,a)$ as $d_{(s,a)}(\widetilde{q})$. We assume zero mean uniform distributions for the following reasons. First, if the mean were different than zero, it could be used to improve the estimate $\widehat{q}(s,a)$ as $\widehat{q}(s,a)\leftarrow\widehat{q}(s,a)+\Ex\bs{\widetilde{q}}(s,a)$ resulting in a new estimate whose corresponding error would be zero mean. Secondly, under assumption 1, we know that for any infinitely discounted MDP, a symmetric bound for the error of the value function exists in the form $-\ell(s,a)<\bs{\widetilde{q}}(s,a)<\ell(s,a)$.\footnote{This is due to the fact that the value functions are lower and upper bounded by $r_{\textrm{min}}/(1-\gamma)$ and $r_{\max}/(1-\gamma)$.} Moreover, typically there is no prior information about the error distribution between these bounds and therefore a non-informative uniform distribution becomes appropriate.

We further define the state uncertainty distribution $u_s^\pi(\widetilde{q})$ (which is given by a mixture of the state-action error distributions) and the Maximum-Uncertainty-Entropy policy $\pi^\bullet$ (which at every state chooses the action whose corresponding $\bs{\widetilde{q}}(s,a)$ has the greatest uncertainty):
\begin{align}\label{eq:bellman_error_dist}
	&u_s^\pi(\widetilde{q})=\Ex_{\pi}d_{(s,\bs{a})}(\widetilde{q}),\hspace{10mm}\pi^\bullet(a|s)=\begin{cases}
	1,\text{    if }\hspace{2mm}\ell(s,a)=\max_{a}\ell(s,a)\\
	0, \text{    else}
	\end{cases}
\end{align}
We thus define the goal of our reinforcement learning agent to be:
\begin{align}\label{eq:optimization_problem}
\pi^\star(a|s)&=\argmax_{\pi}\Ex_{\mathcal{P},\pi}\bigg(\sum_{t=0}^{\infty}\gamma^{t}\big[\bm{r}(\bm{s}_{t},\bm{a}_t,\bm{s}_{t+1})-\kappa D_{KL}(u_{\bm{s}_{t}}^\pi(\widetilde{q})||u_{\bm{s}_{t}}^\bullet(\widetilde{q}))\big]\Big|\bm{s}_0=s\bigg)
\end{align}
where $\kappa$ is a positive regularization parameter, $D_{KL}$ is the Kullback-Leibler divergence, and $u_{\bm{s}_{t}}^\bullet(\widetilde{q})$ is the state uncertainty distribution corresponding to $\pi^\bullet$. In this work we refer to $\pi^\star(a|s)$ as the \textit{uncertainty constrained} optimal policy (or \textit{uc}-optimal policy). Under this new objective, we redefine the value functions as:
\begin{subequations}\label{eq:value_function_redef}
	\begin{align}
	&q^\pi(s,a)=r(s,a)+\gamma\Ex_{\mathcal{P}}v^\pi(\bm{s}')\label{eq:value_function_redef_c}\\
	&v^\pi(s)\hspace{-1mm}=\hspace{-1mm}\Ex\hspace{-0.5mm}\bigg(\hspace{-0.5mm}\sum_{t=0}^{\infty}\gamma^{t}\big[\bm{r}(\bm{s}_{t},\bm{a}_t,\bm{s}_{t+1})\hspace{-1mm}-\hspace{-1mm}\kappa D_{KL}(u_{\bm{s}_{t}}^\pi(\widetilde{q}))\big]\big|\bm{s}_0\hspace{-1mm}=\hspace{-1mm}s\bigg)\hspace{-1mm}=\hspace{-1mm}\Ex\big(r(s,a)\hspace{-1mm}-\hspace{-1mm}\kappa D_{KL}(u_{s}^\pi(\widetilde{q}))\hspace{-1mm}+\hspace{-1mm}\gamma v^\pi(\bs{s'})\big)\label{eq:value_function_redef_b}
	\end{align}
\end{subequations}
where, with a little abuse of notation, we defined $D_{KL}(u_{s}^\pi(\widetilde{q}))\define D_{KL}(u_{s}^\pi(\widetilde{q})||u_{s}^\bullet(\widetilde{q}))$. Using \eqref{eq:value_function_redef} we can rewrite \eqref{eq:optimization_problem} as:
\begin{align}\label{eq:optimization_problem_nice}
\pi^\star&=\argmax_{\pi}\sum_{a}\pi(a|s)\widehat{q}(s,a)-\kappa D_{KL}(u_{s}^\pi(\widetilde{q}))
\end{align}
Note that the exploration-exploitation trade-off becomes explicit in \eqref{eq:optimization_problem_nice}. To maximize the first term of the summation in \eqref{eq:optimization_problem_nice}, the agent has to exploit its knowledge of $\widehat{q}$, while to maximize the second term, the agent's policy needs to match $\pi^\bullet$, which is a policy that seeks to maximize the information gathered through exploration. Since the argument being maximized in \eqref{eq:optimization_problem_nice} is differentiable with respect to $\pi(a|s)$ we can obtain a closed-form expression for $\pi^\star$. Before presenting such closed-form expression we introduce the following useful remark, lemmas and definitions.
\begin{remark}
	In the interest of simplifying equations and notation we clarify that from now on we will use $\ell_i(s)\hspace{-0.5mm}\define\hspace{-0.5mm}\ell(s,a_i)$ and even $\ell_i\hspace{-1mm}\define\hspace{-1mm}\ell(s,a_i)$ when the state $s$ is clear from the context.
\end{remark}
\begin{lemma}\label{lemma:kl_divergence}
	Assuming the actions are ordered such that $\ell_i>\ell_j\iff i>j$, the KL divergence term in \eqref{eq:optimization_problem_nice} for a given state $s$ is given by:
	\begin{align}\label{eq:dkl}
	&\sum_{n=1}^{A}(\ell_{n}\hspace{-0.5mm}-\hspace{-0.5mm}\ell_{n-1})\left(\sum_{k=n}^A\frac{\pi(a_k|s)}{\ell_k}\right)\hspace{-0.5mm}\log\hspace{-0.7mm}\left[\sum_{b=n}^{A}\frac{\pi(b|s)\ell_A}{\ell_{b}}\right]
	\end{align}
\end{lemma}
\begin{proof}
	The proof is included in Appendix 6.1.%\ref{appendix:kl_divergence}.
\end{proof}
\begin{definition}\label{definition:dominated}
	\textbf{Pareto dominated action}: For a certain state $s$ we say that an action $a_j$ is Pareto dominated by action $a_i$ if $\widehat{q}(s,a_j)\leq\widehat{q}(s,a_i)$ and $\ell(s,a_j)<\ell(s,a_i)$.
\end{definition}
\begin{definition}\label{definition:mixed_dominated}
	\textbf{Mixed Pareto dominated action}: For a certain state $s$ we say that an action $a_k$ is mixed Pareto dominated if there exist two actions $a_i$ and $a_j$ that satisfy:
	\begin{align}\label{eq:required_optimality}
		&\ell(s,a_i)>\ell(s,a_k)>\ell(s,a_j),\hspace{10mm}1<\frac{\left(\ell_i-\ell_k\right)\ell_j\widehat{q}(s,a_j)+\left(\ell_k-\ell_j\right)\ell_i\widehat{q}(s,a_i)}{\left(\ell_i-\ell_j\right)\ell_k\widehat{q}(s,a_k)}
	\end{align}
\end{definition}
\begin{definition}
	\textbf{Pareto optimal action}: We define an action $a$ as Pareto optimal if it is not Pareto dominated or mixed Pareto dominated.
\end{definition}
Intuitively (mixed) Pareto dominated actions are actions that the agent should not choose because there is another action (or group of actions) that is better in terms of both exploration and exploitation.
\begin{lemma}\label{lemma:Pareto_dominated}
	For all Pareto dominated and mixed Pareto dominated actions it holds that $\pi^\star(a|s)=0$.
\end{lemma}
\begin{proof}
	The proof is included in Appendix 6.2.%\ref{appendix:proof_pareto_dominated}.
\end{proof}
The statement of Lemma \ref{lemma:Pareto_dominated} is intuitive since choosing a Pareto dominated action lowers the expected cumulative reward and the information gained, relative to choosing the action that dominates it. Also note that Lemma \ref{lemma:Pareto_dominated} implies that for all Pareto optimal actions it must be the case that if $q(s,a_i)<q(s,a_j)$ then $\ell(s,a_i)>\ell(s,a_j)$. We now introduce the state dependent set of actions $\mathcal{E}_s$ with cardinality $|\mathcal{E}_s|$, which is formed by all the Pareto optimal actions corresponding to state $s$. Furthermore, we introduce the ordering functions $\sigma_s(k):[1,|\mathcal{E}_s|]\rightarrow[1,A]$, which for every state provide an ordering amongst the Pareto optimal actions from lowest uncertainty to highest (i.e., $\ell(s,\sigma_s(i))>\ell(s,\sigma_s(j))\iff i>j$). For instance, $\sigma_s(1)$ provides the index of the action at state $s$ that has the lowest uncertainty amongst the actions contained in $\mathcal{E}_s$.
\begin{theorem}\label{lemma:opt_policy}
	\allowdisplaybreaks
	$\pi^\star(a|s)$ is given by:
	\begin{subequations}\label{eq:optimal_policy}
		\begin{align}
		&\pi^\star(a|s)\hspace{-1mm}=\hspace{-1mm}\begin{cases}
		\frac{\ell_{\sigma(j)}\left(p_j(s)-p_{j+1}(s)\right)}{\sum\limits_{j=1}^{|\mathcal{E}_s|}(\ell_{\sigma(j)}-\ell_{\sigma(j-1)})p_{j}(s)},\hspace{3mm}\textrm{if $a=\sigma_s(j)\hspace{1mm}$ for some $j\in[1,|\mathcal{E}_s|]$}\\
		0,\hspace{32mm}\textrm{otherwise}
		\end{cases}\label{eq:optimal_policy_a}\\
		&p_j(s)=\mathrm{exp}\hspace{-0.8mm}\left[\frac{\ell_{\sigma(j)}(s)\widehat{q}(s,\sigma(j))\hspace{-0.5mm}-\hspace{-0.5mm}\ell_{\sigma(j\hspace{-0.4mm}-\hspace{-0.4mm}1)}(s)\widehat{q}(s,{\sigma(j\hspace{-1mm}-\hspace{-1mm}1)})}{\kappa\left(\ell_{\sigma(j)}(s)-\ell_{\sigma(j-1)}(s)\right)}\hspace{-0.4mm}\right]
		\end{align}
	\end{subequations}
	where to simplify notation we defined $\ell_{\sigma(j)}(s)=\ell_{\sigma_s(j)}(s)$ and $\widehat{q}(s,{\sigma(j)})=\widehat{q}(s,\sigma_s(j))$ and we also set $p_{|\mathcal{E}_s|+1}(s)=0$, $l_0(s)=0$.
\end{theorem}
\begin{proof}
	The proof is included in Appendix 6.3.
\end{proof}
Note that as expected, $\pi^\star(a|s)$ is always strictly positive for Pareto optimal actions.
\begin{lemma}\label{lemma:opt_value_function}
	The value function corresponding to policy $\pi^\star(a|s)$ is given by:
	\begin{align}\label{eq:optimal_value_function}
		&v^\star(s)=\kappa\log\left[\sum_{j=1}^{|\mathcal{E}_s|}\frac{\ell_{\sigma(j)}(s)-\ell_{\sigma(j-1)}(s)}{\ell_{\max}(s)}p_{j}(s)\right],\hspace{5mm}q^\star(s,a)=r(s,a)+\gamma\Ex v^\star(\bs{s'})
	\end{align}
	where $\ell_{\max}(s)=\max_{a}\ell(s,a)$.
\end{lemma}
\begin{proof}
	The proof follows by combining \eqref{eq:optimal_policy} with \eqref{eq:value_function_redef_b} and \eqref{eq:value_function_redef_c}.
\end{proof}
\begin{remark}\label{remark:limit}
	$\pi^\star(a|s)$ satisfies the following conditions:
	\begin{align}
	&\lim_{\kappa\rightarrow 0^+}\pi^\star(a|s)=\pi^\dagger(a|s),\hspace{5mm}\lim_{\ell_A\rightarrow \ell_{A-1}^+\cdots\rightarrow\ell_1^+}\pi^\star(a|s)=\pi^\dagger(a|s)\label{eq:consistent_2}
	\end{align}
\end{remark}
The first condition is expected since when the relative entropy term is eliminated, \eqref{eq:opt_policy} and \eqref{eq:optimization_problem} become equivalent. The second condition reflects the fact that when the uncertainty is equal for all actions, the distributions $u_s^\pi(\widetilde{q})$ and $u_s^\bullet(\widetilde{q})$ become equal regardless of $\pi$ and therefore $D_{KL}(u_s^\pi(\widetilde{q}))=0$ and hence \eqref{eq:opt_policy} and \eqref{eq:optimization_problem} become equivalent. The second condition of \eqref{eq:consistent_2} is of fundamental importance because it guarantees that as the uncertainties over $\widehat{q}$ diminish (and therefore converge to some required threshold $\ell(s,a)\rightarrow\ell_{\min}$), policy $\pi^\star$ tends to the desired policy $\pi^\dagger$ (note that annealing of $\kappa$ is not necessary for this convergence of $\pi^\star$ towards $\pi^\dagger$). Note that policy \eqref{eq:optimal_policy} has the previously discussed qualities that policies induced by maximum entropy RL do not. Namely, as learning progresses, the effect of the regularizer also diminishes, and hence so does the bias of $\pi^\star$ with respect to $\pi^\dagger$, without the need for annealing $\kappa$. Furthermore, the effect of the regularizer diminishes over time on a per state basis, allowing for a high degree of exploration in states where the agent has high uncertainty and high exploitation in states were the state has high certainty over its estimates. 

\subsection{Uncertainty Constrained Value Iteration}
In a dynamic programing setting $q^\star(s,a)$ can be found by iteratively applying to any vector $q(s,a)$ the operator $\mathcal{T}^\ell$ defined by:
\begin{align}\label{eq:q_backup}
&\mathcal{T}^\ell q(s,a)=r(s,a)+\gamma\kappa\Ex\log\Bigg[\sum_{j=1}^{|\mathcal{E}_s|}\frac{\ell_{\sigma(j)}(\bs{s'})-\ell_{\sigma(j-1)}(\bs{s'})}{\ell_{\max}(\bs{s'})}p_{j}(\bs{s'})\Bigg]
\end{align}
where we make the $\ell$ explicit to highlight the fact that $q^\star(s,a)$ is a function of the uncertainties.
\begin{lemma}\label{lemma:policy_eval}
	\textbf{$\bs{\ell}$-Policy Evaluation}: For any mapping $q^0:\mathcal{S}\times\mathcal{A}:\rightarrow\mathbb{R}$, the sequence $q^{n+1}=\mathcal{T}^\ell q^n$ converges to $q^\star$.
\end{lemma}
\begin{proof}
	The proof follows by noting that $\mathcal{T}^\ell$ is a contraction mapping and applying Banach's fixed point theorem. See Appendix 6.4.%\ref{appendix:lemma_policy_eval}.
\end{proof} 
Note that Lemma \ref{lemma:policy_eval} provides an algorithm to learn $q^\star$, however the ultimate goal of the RL agent is to learn $q^\dagger$. To accomplish this goal a mechanism to estimate $\ell$ is necessary.

\subsection{Uncertainty Estimation}
Recalling that $\widetilde{q}(s,a)=q^\dagger(s,a)-\widehat{q}(s,a)$ we can write:
\begin{align}
\widetilde{q}(s,a)&=r(s,a)+\gamma\Ex_{\bm{s}'}v^\dagger(\bm{s}')-\widehat{q}(s,a)=\Ex\bs{\delta}(s,a,\bs{s'})+\gamma\Ex_{\bm{s}'}\widetilde{v}(\bm{s}')\label{eq:error_1}
\end{align}
where $\bs{\delta}(s,a,\bs{s'})=\bs{r}(s,a,\bs{s'})+\gamma v(\bs{s'})-\widehat{q}(s,a)$ and $v^\dagger(s)=\widehat{v}(s)+\widetilde{v}(s)$. Furthermore $\widehat{v}(s)$ is the estimate obtained using \eqref{eq:optimal_value_function} and $\widehat{q}(s,a)$. We can now bound $\widetilde{v}(s')$ as follows:
\begin{align}
&v^\dagger(s)-\widehat{v}(s)=\max_aq^\dagger(s,a)-\widehat{v}(s)\stackrel{(a)}{\leq}\max_aq^\dagger(s,a)-\max_a\widehat{q}(s,a)\stackrel{(b)}{\leq}\max_{a}\ell(s,a)\label{eq:error_2}
\end{align}
where $(a)$ follows from Jensen's inequality applied to $v^\star(s)$ and in $(b)$ we applied \eqref{eq:l_condition}. Combining \eqref{eq:error_1} and \eqref{eq:error_2} we get $|\widetilde{q}(s,a)|\leq|\Ex\bs{\delta}(s,a,\bs{s'})|+\gamma\Ex\max_{a}\ell(\bm{s}',a)$. Therefore updating $\ell(s,a)$ as:
\begin{align}\label{eq:l_backup}
\ell(s,a)\leftarrow|\Ex\bs{\delta}(s,a,\bs{s'})|+\gamma\Ex\max_{a}\ell(\bm{s}',a)
\end{align}
guarantees that condition \eqref{eq:l_condition} is satisfied.
\begin{lemma}\label{lemma:l_eval}
	\textbf{Uncertainty Constrained Policy Evaluation}: For any mapping $q^0:\mathcal{S}\times\mathcal{A}:\rightarrow\mathbb{R}$, repeated application of $\ell$-Policy Evaluation and \eqref{eq:l_backup} converges to $q^\dagger$.
\end{lemma}
\begin{proof}
	See Appendix 6.5.%\ref{appendix:l_eval}. 
\end{proof}
\subsection{Information Seeking Learner}
We now proceed to derive a practical approximation to Uncertainty Constrained Policy Evaluation. We start by noting that \eqref{eq:l_backup} is not adequate to design a stochastic algorithm because the $|\Ex\bs{\delta}(s,a,\bs{s'})|$ term has an expectation inside the absolute value operator. Therefore a stochastic approximation of the form $|\delta(s_t,a_t,s_{t+1})|$ would be biased (since the sample approximation $|\delta(s_t,a_t,s_{t+1})|$ approximates $\Ex|\bs{\delta}(s,a,\bs{s'})|$ instead of $|\Ex\bs{\delta}(s,a,\bs{s'})|$). In the particular case where the MDP is deterministic this is not an issue because $\bs{\delta}(s,a,\bs{s'})$ is a deterministic quantity and therefore can be calculated with any sample transition $(s,a,r,s')$. But in the general case, we can have an estimator $\widehat{\rho}(s,a)$ of $\Ex\bs{\delta}(s,a,\bs{s'})$, which then can be used to estimate $|\Ex\bs{\delta}(s,a,\bs{s'})|$ as $|\rho(s,a)|$. With this consideration and equations \eqref{eq:q_backup} and \eqref{eq:l_backup} we can define the update equation for the tabular version of the algorithm we present in this work as:
\begin{subequations}\label{eq:tabular}
	\begin{align}
	&q(s_t,a_t)=q(s_t,a_t)+\mu_{q}\left(\delta(s_t,a_t,s_{t+1})\right)\\
	&\rho(s_t,a_t)=\rho(s_t,a_t)+\mu_{\rho}\left(\delta(s_t,a_t,s_{t+1})-\rho(s_t,a_t)\right)\\
	&\ell(s_t,a_t)=\ell(s_t,a_t)+\mu_{\ell}\big((1-\eta_1)|\delta(s_t,a_t,s_{t+1})|+\eta_1|\rho(s_t,a_t)|+\gamma\ell_{\max}(s_{t+1} )-\ell(s_t,a_t)\big)\\
	&\delta(s_t,a_t,s_{t+1})=r(s_t,a_t,s_{t+1})+\gamma\widehat{v}(s_{t+1})-\widehat{q}(s_t,a_t)
	\end{align}
\end{subequations}
where $\eta_1$ is a tunable hyperparameter. In cases where the MDP is deterministic then $\eta_1=0$. For stochastic games $\eta_1$ closer to 1 becomes more convenient. As we mentioned before we want to derive an approximation to Uncertainty Constrained Policy Evaluation suitable for practical applications, which typically require that the $q$, $\rho$ and $\ell$ functions are parameterized using expressive function approximators such as neural networks (NNs). In this work we use the parameters $\omega$, $\theta$ and $\nu$ to parameterize $q$, $\rho$ and $\ell$, respectively. To extend \eqref{eq:tabular} to this general case, $\omega$, $\theta$ and $\nu$ can be trained to minimize \eqref{eq:costs}.
\begin{subequations}\label{eq:costs}
	\allowdisplaybreaks
	\thinmuskip=0mu
	\medmuskip=0mu
	\thickmuskip=0mu
	\begin{align}
	&J_\rho(\theta)=2^{-1}\Ex_{(\bs{s},\bs{a})\sim\psi}\left[\delta(\bs{s},\bs{a};\omega)-\rho(\bs{s},\bs{a};\theta)\right]^2\\
	%	&J_q(\omega)=2^{-1}\Ex_{(\bs{s},\bs{a})\sim\psi}\big[\hat{q}(\bs{s},\bs{a};\omega)-q_T(\bs{s},\bs{a};\widetilde{\omega})\big]^2\\
	&J_q(\omega)=2^{-1}\Ex_{(\bs{s},\bs{a})\sim\psi}\big[(q_T(\bs{s},\bs{a};\widetilde{\omega})-\hat{q}(\bs{s},\bs{a};\omega))((1-\eta_2)(q_T(\bs{s},\bs{a};\widetilde{\omega})-\hat{q}(\bs{s},\bs{a};\omega))+\eta_2\rho(\bs{s},\bs{a};\widetilde{\theta}))\big]\\
	&J_\ell(\nu)=2^{-1}\Ex_{(\bs{s},\bs{a})\sim\psi}[\ell_T(\bs{s},\bs{a};\widetilde{\nu})-\ell(\bs{s},\bs{a};\nu)]^2\\
	&q_T(\bs{s},\bs{a};\widetilde{\omega})=r(\bs{s},\bs{a})+\kappa\Ex\log\left[\sum_{j=1}^{|\mathcal{E}_s|}\frac{(\ell_{\sigma(j)}(\bs{s'})-\ell_{\sigma(j-1)}(\bs{s'}))}{\ell_{\max}(\bs{s'})}p_j(\bs{s'})\right]\\
	&\ell_T(s,a;\widetilde{\nu})=(1-\eta_1)|\delta(s,a)|+\eta_1|\rho(s,a)|+\gamma\Ex_{\bm{s}'}\ell_{\max}(\bm{s}';\widetilde{\nu})\label{eq:l_T}
	\end{align}
\end{subequations}
where $\psi$ is the distribution according to which the $(\bs{s},\bs{a})$ pairs are sampled, and $\widetilde{\omega}$ and $\widetilde{\nu}$ are used to denote the parameters of the target networks corresponding to $q$ and $\ell$, respectively. Note also that we have added another tunable parameter $\eta_2$. In the tabular case this is not necessary, but in the case with neural networks we observed empirically that using the $\rho$ network to train $q$ helps stabilize training and improves performance. This is similar to the SAC algorithm where a network is used to estimate the value function $v(s)$ even though apparently it is not necessary \cite{sac}. The resulting algorithm, which we refer to as \textit{Information Seeking Learner}, is listed in Algorithm 1.
\begin{algorithm}[t!]
	\caption{Information Seeking Learner (ISL)}
	\label{alg:isl}
	\begin{algorithmic}
		\STATE{\bfseries Initialize:} counter=0, $\omega$, $\theta$ and $\nu$ randomly, and an empty replay buffer $\mathcal{D}$.%and $\ell(s,a)=\max\{r_{\max},|r_{\textrm{min}}|\}(1-\gamma)^{-1}$ for all $(s,a)$.
		\FOR{iterations $k=0,\ldots,K$}
		\FOR{environment transitions $t=0,\ldots,T$}
		\STATE Sample transitions $(s,a,r,s')$ by following policy \eqref{eq:optimal_policy} and store them in $\mathcal{D}$.
		\ENDFOR
		\FOR{iterations $i=0,\ldots,I$}
		\STATE Sample a minibatch from $\mathcal{D}$ and compute stochastic gradients.
		\STATE $\omega\leftarrow\omega-\mu_\omega\widehat{\nabla}_\omega J_q(\omega),\hspace{4mm}\theta\leftarrow\theta-\mu_\theta\widehat{\nabla}_\theta J_\rho(\theta),\hspace{4mm}\nu\leftarrow\nu-\mu_\nu\widehat{\nabla}_\nu J_\ell(\nu)$
		\STATE counter += 1
		\IF{counter \textbf{mod} targetUpdatePeriod}
		\STATE $\widetilde{\nu}\leftarrow\nu,\hspace{4mm}\widetilde{\omega}\leftarrow\omega$
		\STATE counter = 0
		\ENDIF
		\ENDFOR
		\ENDFOR
	\end{algorithmic}
\end{algorithm}

\section{Experiments}\label{section:Experiments}
The goal of our experiments is to evaluate the capacity of ISL to perform deep exploration. We test our algorithm in the three deep exploration environments provided by the \textit{bsuite} benchmark \cite{bsuite} (i.e., Cartpole Swingup, Deep Sea and Deep Sea Stochastic) and compare it against SBEED, UBE \cite{ube} and Bootstrap DQN with prior networks (BSP)\footnote{We use the implementation provided by \cite{bsuite}.} (which provides state of the art results in deep exploration tasks). All three environments have the common quality that exploration is discouraged (due to negative rewards) and positive rewards are only obtained by the agent in states that are hard to reach. We used NNs as function approximators in all cases. Implementation details are included in  the Appendix. In appendix 6.9 we also include an ablation study for hyperparameters $\kappa$, $\eta_1$ and $\eta_2$.

\subsection{Sparse Cartpole Swingup}
This is the classical cartpole swingup task with the added difficulty that positive rewards are only provided when the pole is `almost stabilized'. The action space is \{\textit{left, stay, right}\} and the state space is continuous and given by $s_t=(\cos(\theta_t),\sin(\theta_t),\dot{\theta}_t,x_t,\dot{x}_t)$, where $\theta$ is the angle of the pole and $x$ is the position of the cart. The feature that makes this task a challenging exploration task is the reward structure; every move is penalized with a $-0.1$ reward and a $+1$ is only observed when the cart is `almost' centered and the pole is `almost' upright and stabilized. More specifically, a $+1$ reward is obtained when $cos(\theta)>N/20$, $|\dot{\theta}|<1$ and $|x|<1-N/20$, where $N$ parameterizes the difficulty of the environment. This is an episodic task where each episode ends when the cart moves too far away from the center ($|x|>3$) or at $10^3$ time-steps, whichever occurs first. We ran each algorithm for $10^3$ episodes for ten random seeds. In this benchmark the performance measure is the best return attained during training. The results for $N$ from 0 to 19 are shown in figures \ref{fig:cartpole_best}. As expected SBEED fails in this task for all values of $N$ due to the lack of a mechanism to encourage deep exploration. ISL and BSP find close to optimal policies for all values of $N$ up to 10 approximately. However, for higher values of $N$ ISL outperforms BSP by a significant margin. For $N=19$ all algorithms fail, however ISL is the only algorithm who can obtain some positive rewards.
\begin{figure*}[t]
	%\vskip 0.2in
	\begin{center}
		\centering
		\begin{subfigure}{0.32\textwidth}
			\includegraphics[width=\textwidth]{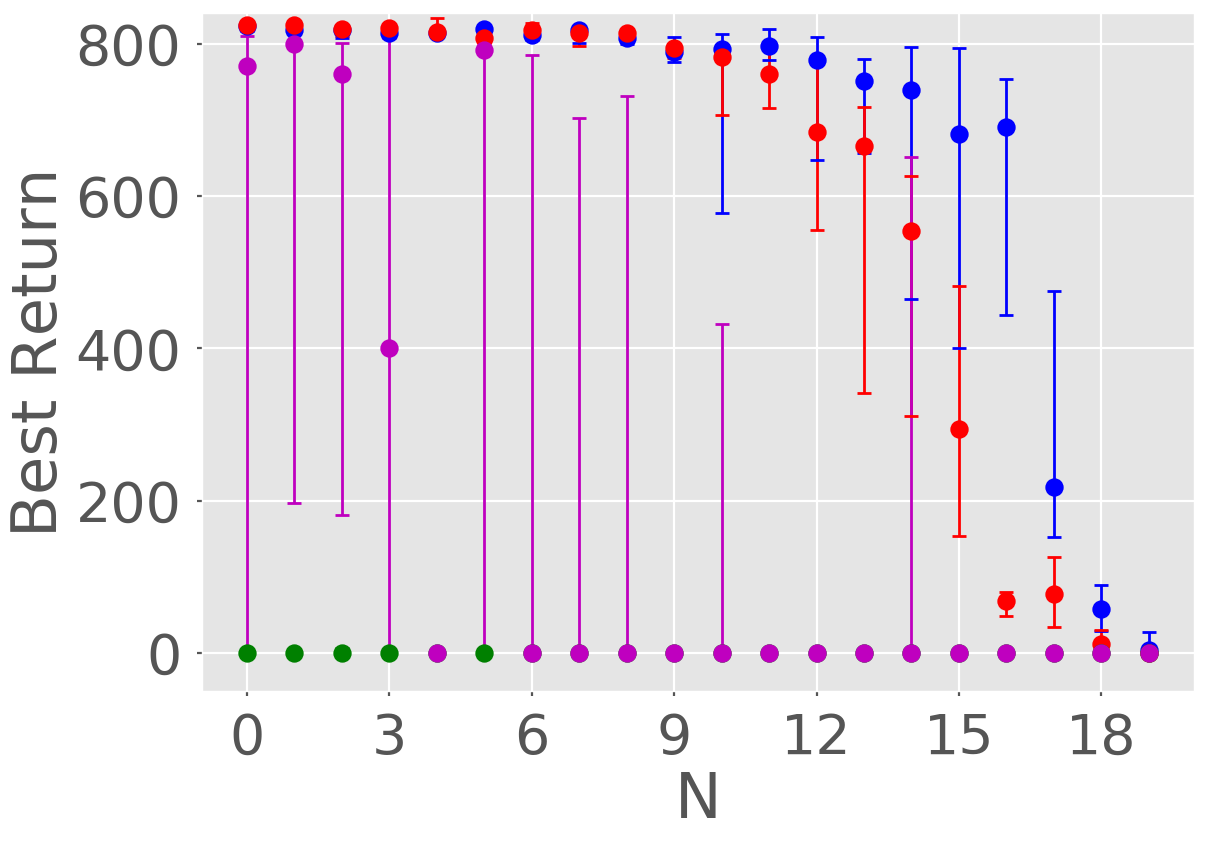}
			\subcaption{Cartpole swingup}
			\label{fig:cartpole_best}
		\end{subfigure}
		\begin{subfigure}{0.32\textwidth}
			\includegraphics[width=\textwidth]{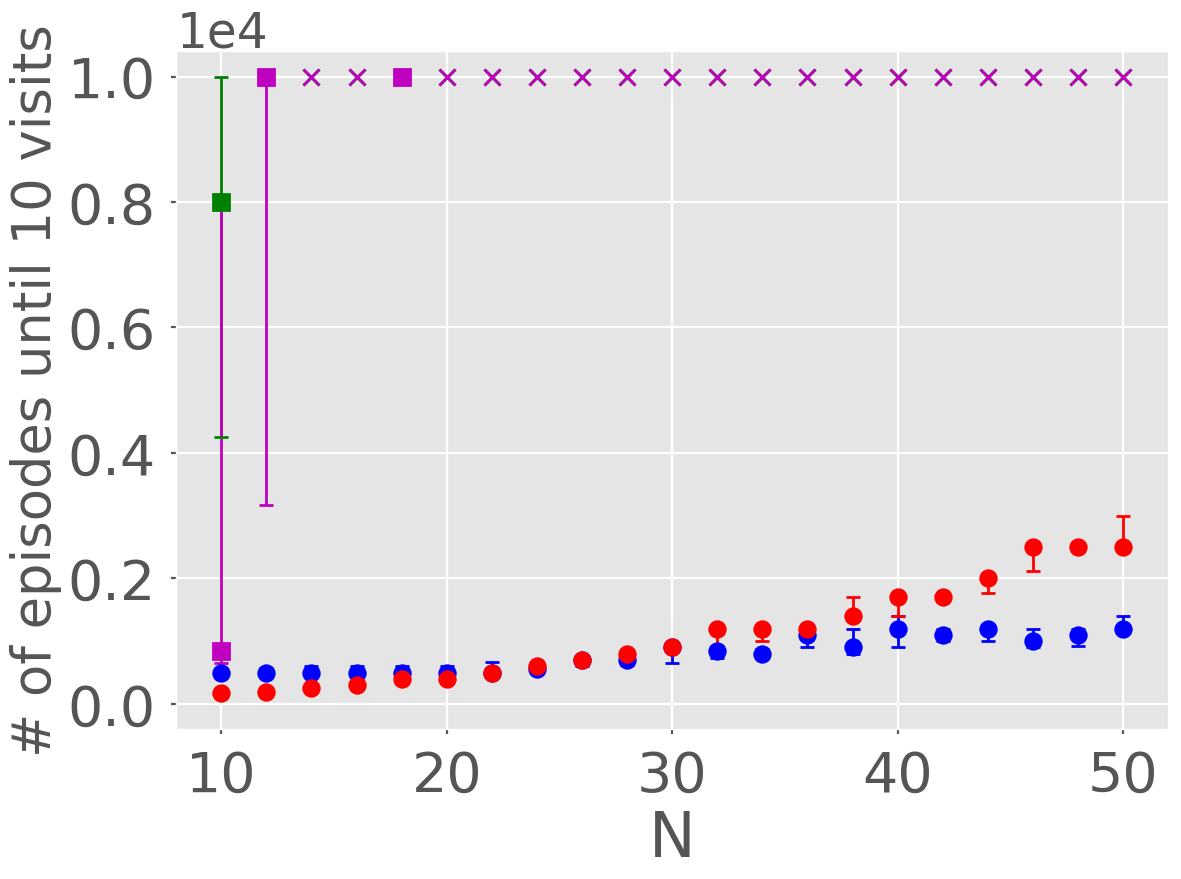}
			\subcaption{Deep sea}
			\label{fig:deep_sea_success}
		\end{subfigure}
		\begin{subfigure}{0.32\textwidth}
			\includegraphics[width=\textwidth]{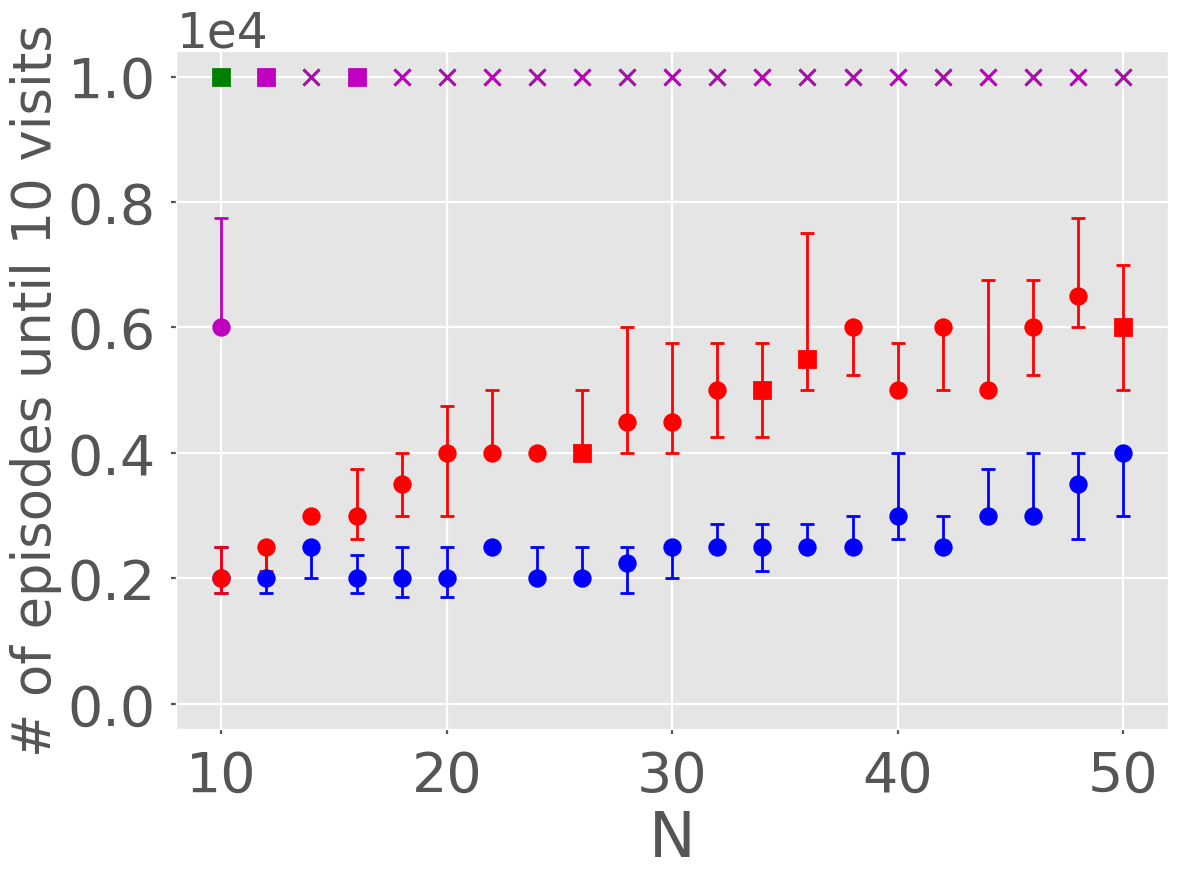}
			\subcaption{Deep sea stochastic}
			\label{fig:deep_sea_stoch_success}
		\end{subfigure}
		\caption{Blue, red, purple and green curves correspond to ISL, BSP, UBE and SBEED, respectively. In all cases we ran 10 experiments with different seeds, the plots show the median and first and third quartiles. In figures \ref{fig:deep_sea_success} and \ref{fig:deep_sea_stoch_success} we used dots are markers when the goal was accomplished (at least 10 visits where made to the desired state) for all seeds, square markers denote that the goal was accomplished for some seeds and the cross markers denote failure for all seeds.}
		\label{fig:figures}
	\end{center}
	%\vskip -0.2in
\end{figure*}

\subsection{Deep Sea and Deep Sea Stochastic}
Deep Sea is an episodic grid-world type game which consists of an $N\times N$ grid with $N^2$ states. The observation encodes the agent's row and column as a one-hot vector $s_t\in\{0,1\}^{N\times N}$. The environment has two possible actions $\{a_1,a_2\}$ and a mask $M\sim\text{Ber}(0.5)^{N\times N}$. The mask maps for every state each action to $\{\textit{left, right}\}$. The agent always starts at the top-left corner and at every step deterministically descends one level and further moves left or right (depending on the chosen action). Every time the agent moves right it gets a $-0.01/N$ reward, except for the bottom-right state in which case it gets a $+1$, while left moves always receive $0$ reward. The game ends after $N$ time steps and we ran each algorithm for $10^4$ episodes. The optimal strategy of the environment is to always move right in which case the total return would be $0.99$. Note that the parameter $N$ parameterizes the difficulty of the game. An important point about this environment is that policies that do not encourage deep exploration take $\mathcal{O}(2^N)$ episodes to learn the optimal policy \cite{osband2016generalization}, while for policies that explore optimally it takes at best $\mathcal{O}(N)$ episodes (because the state-action space is of size $\mathcal{O}(N^2)$ and at every episode $N$ state-action pairs are sampled). The Deep Sea Stochastic environment has the added complexity that transitions and rewards are stochastic. In particular, the reward of the last step of the episode is corrupted with additive Gaussian noise with zero mean and variance equal to 1 and further, agents transition to the right only with $1-1/N$ probability. In these benchmarks the performance measure is the amount of episodes played before the agent visits the goal state for the 10-th time. Hyperparameters were tuned for each of the cases, figures \ref{fig:deep_sea_success} and \ref{fig:deep_sea_stoch_success} show the results. Again SBEED fails at these tasks as expected. In the deterministic case ISL and BSP perform similarly. Note also that while ISL shows linear complexity for all values of $N$, BSP shows linear complexity only for low values of $N$. We clarify that the complexity of BSP could probably be improved for larger values of $N$ by enlarging the ensemble size, however note that this would come with an added computational cost since the computational cost of BSP scales linearly with the size of the ensemble. In the stochastic environment ISL performs similarly as it does in the deterministic environment and outperforms BSP for all values of $N$. Note further that ISL is the only algorithm that is able to solve the task for all values of $N$ for all seeds.

\section{Conclusion}
In this paper we provide a novel and original approach to address the problem of deep exploration. We also make an interesting connection with the literature on maximum entropy RL. In contrast to current RL algorithms and deep exploration strategies, where deriving the learning equations and the deep exploration strategies are treated separately, in our approach, both the learning equations and the deep exploration strategy are derived in tandem as the solution to a unique optimization problem. Furthermore, we have introduced a practical way of quantifying uncertainty over $q$ estimates that is usable with NNs. We hope this ideas might inspire novel and original research directions.

\section*{Broader Impact}

We believe the material we present does not introduce any societal or ethical considerations worth mentioning in this section.

\bibliography{isl2020}
%Uncomment the following two lines to include appendix
\newpage

\onecolumn
\section{Appendix}
\thinmuskip=0mu
\medmuskip=0mu
\thickmuskip=0mu
\subsection{Proof of Lemma \ref{lemma:kl_divergence}}\label{appendix:kl_divergence}
In the proof we assume that the actions are ordered following the lemma's assumption. By definition, the KL divergence is given by:
\begin{align}
	D_{KL}(u_{s}^\pi(\widetilde{q}))&=\int_{\widetilde{q}}u_{s}^\pi(\widetilde{q})\log\left(\frac{u_{s}^\pi(\widetilde{q})}{u_{s}^\bullet(\widetilde{q})}\right)d\widetilde{q}\nonumber\\
	&\stackrel{(b)}{=}\sum_{a}\pi(a|s)\int_{\widetilde{q}}d_{(s,a)}(\widetilde{q})\log\left(\sum_{a'}\pi(a'|s)d_{(s,a')}(\widetilde{q})\right)d\widetilde{q}+\log(\ell_A)\label{eq:Dkl}
\end{align}
where in $(b)$ we used \eqref{eq:bellman_error_dist}. Note that since $\delta^\pi(s)$ is a piecewise constant distribution, the integral in \eqref{eq:Dkl} has the following closed form expression:
\begin{align}
	&\int_{\widetilde{q}}d_{(s,a_j)}(\widetilde{q})\log\left(\sum_{a'}\pi(a'|s)d_{(s,a')}(\widetilde{q})\right)d\widetilde{q}=\int_{-\ell_j}^{\ell_j}(2\ell_j)^{-1}\log\left(\sum_{a'}\pi(a'|s)d_{(s,a')}(\widetilde{q})\right)d\widetilde{q}\nonumber\\
	&=\ell_j^{-1}\int_{{\ell_{j-1}}}^{\ell_j}\log\left(\sum_{a'}\pi(a'|s)d_{(s,a')}(\widetilde{q})\right)d\widetilde{q}+\ell_j^{-1}\int_{0}^{\ell_{j-1}}\log\left(\sum_{a'}\pi(a'|s)d_{(s,a')}(\widetilde{q})\right)d\widetilde{q}\nonumber\\
	&=\frac{\ell_{j}-\ell_{j-1}}{\ell_{j}}\log\left(\sum_{b=0}^{A-j}\pi(A-b|s)\ell_{A-b}^{-1}\right)+\ell_j^{-1}\int_{0}^{\ell_{j-1}}\hspace{-5mm}\log\left(\sum_{a'}\pi(a'|s)d_{(s,a')}(b)\right)d\widetilde{q}\nonumber\\
	&=\sum_{n=1}^{j}\frac{\ell_{n}-\ell_{n-1}}{\ell_{j}}\log\left(\sum_{b=n}^{A}\frac{\pi(b|s)}{\ell_{b}}\right)\label{eq:integral}
\end{align}
Combining \eqref{eq:Dkl} and \eqref{eq:integral} we get:
\begin{align}\label{eq:dkl3}
&\sum_{k=1}^A\frac{\pi(a_k|s)}{\ell_k}\sum_{n=1}^{k}(\ell_{n}-\ell_{n-1})\log\left[\sum_{b=n}^{A}\frac{\pi(b|s)\ell_A}{\ell_{b}}\right]
\end{align}
Rearranging the terms in \eqref{eq:dkl3} we get:
\begin{align}\label{eq:app_dkl2}
&D_{KL}(u_s^\pi(\widetilde{q}))=\sum_{n=1}^{A}(\ell_{n}-\ell_{n-1})\left(\sum_{k=n}^A\frac{\pi(a_k|s)}{\ell_k}\right)\log\left[\sum_{b=n}^{A}\frac{\pi(b|s)\ell_A}{\ell_{b}}\right]
\end{align}
\subsection{Proof Lemma \ref{lemma:Pareto_dominated}}\label{appendix:proof_pareto_dominated}
We start proving that for any action $a_j$ that is Pareto dominated by another action $a_i$ it must be the case that $\pi^\star(a_j|s)=0$. We now present an assumption to make notation simpler.
\begin{assumption}
	In this section we assume without loss of generality that actions are ordered such that $\ell_i>\ell_j\iff i>j$.
\end{assumption}
We prove the lemma by contradiction. Assume that there is a $uc$-optimal policy $\pi_1$ for which $\pi_1(a_j|s_1)>0$. We now define policy $\pi_2$ as:
\begin{align}
&\pi_2(a|s)=\begin{cases}
\pi_1(a|s)-\alpha\hspace{5mm}\textrm{if $(s,a)=(s_1,a_j)$}\\
\pi_1(a|s)+\alpha\hspace{5mm}\textrm{if $(s,a)=(s_1,a_i)$}\\
\pi_1(a|s)\hspace{10mm}\textrm{else}
\end{cases}
\end{align}
where $0<\alpha<\pi_1(a|s)$. We show that $\frac{\partial v^{\pi_2}(s_1)}{\partial\alpha}\big|_{\alpha=0}>0$ and hence $v^{\pi_2}(s_1)>v^{\pi_1}(s_1)$ for a small enough $\alpha>0$, which contradicts the claim that $\pi_1$ is a $uc$-optimal policy.
\begin{align}
	\frac{\partial v^{\pi_2}(s_1)}{\partial\alpha}\bigg|_{\alpha=0}&=\frac{\partial v^{\pi_2}(s_1)}{\partial\pi_1(a_j|s_1)}\frac{\partial \pi_1(a_j|s_1)}{\partial\alpha}\bigg|_{\alpha=0}+\frac{\partial v^{\pi_2}(s_1)}{\partial\pi_1(a_i|s_1)}\frac{\partial \pi_1(a_i|s_1)}{\partial\alpha}\bigg|_{\alpha=0}\nonumber\\
	&=\frac{\partial v^{\pi_2}(s_1)}{\partial\pi_1(a_i|s_1)}-\frac{\partial v^{\pi_2}(s_1)}{\partial\pi_1(a_j|s_1)}\nonumber\\
	&=\underbrace{\widehat{q}^{\pi_2}(s,a_i)-\widehat{q}^{\pi_2}(s,a_j)}_{>0\textrm{ (due to Pareto assumption)}}+\kappa\left(\frac{\partial D_{KL}(u_{s_1}^{\pi_2}(\widetilde{q}))}{\partial\pi_1(a_j|s_1)}-\frac{\partial D_{KL}(u_{s_1}^{\pi_2}(\widetilde{q}))}{\partial\pi_1(a_i|s_1)}\right)\label{eq:diff_kls}
\end{align}
Using \eqref{eq:app_dkl2} we get the following expression for the gradient of the KL term:
\begin{align}\label{eq:grad_kl}
\frac{\partial D_{KL}(u_s^\pi(\widetilde{q}))}{\partial\pi(a_j|s)}&=\sum_{b=1}^j\frac{(\ell_b-\ell_{b-1})}{\ell_j}\log\left(\sum_{c=b}^{A}\frac{\pi(c|s)}{\ell_c}\right)+\ell_j^{-1}\sum_{b=1}^j(\ell_b-\ell_{b-1})\nonumber\\
&=\sum_{b=1}^j\frac{(\ell_b-\ell_{b-1})}{\ell_j}\log\left(\sum_{c=b}^{A}\frac{\pi(c|s)}{\ell_c}\right)+1
\end{align}
Combining \eqref{eq:grad_kl} and \eqref{eq:diff_kls} we get:
\begin{align}
&\frac{\partial D_{KL}(u_{s_1}^{\pi_2}(\widetilde{q}))}{\partial\pi_1(a_j|s_1)}-\frac{\partial D_{KL}(u_{s_1}^{\pi_2}(\widetilde{q}))}{\partial\pi_1(a_i|s_1)}=\sum_{b=1}^j\frac{(\ell_b-\ell_{b-1})}{\ell_j}\log\left(\sum_{c=b}^{A}\frac{\pi(c|s)}{\ell_c}\right)\nonumber\\
&\hspace{50mm}-\sum_{b=1}^i\frac{(\ell_b-\ell_{b-1})}{\ell_i}\log\left(\sum_{c=b}^{A}\frac{\pi(c|s)}{\ell_c}\right)\nonumber\\
&=\sum_{b=1}^j\frac{(\ell_b-\ell_{b-1})(\ell_i-\ell_j)}{\ell_j\ell_i}\log\left(\sum_{c=b}^{A}\frac{\pi(c|s)}{\ell_c}\right)-\sum_{b=j+1}^i\frac{(\ell_b-\ell_{b-1})}{\ell_i}\log\left(\sum_{c=b}^{A}\frac{\pi(c|s)}{\ell_c}\right)\nonumber\\
&\stackrel{(a)}{>}\sum_{b=1}^j\frac{(\ell_b-\ell_{b-1})(\ell_i-\ell_j)}{\ell_j\ell_i}\log\left(\sum_{c=j}^{A}\frac{\pi(c|s)}{\ell_c}\right)-\sum_{b=j+1}^i\frac{(\ell_b-\ell_{b-1})}{\ell_i}\log\left(\sum_{c=j}^{A}\frac{\pi(c|s)}{\ell_c}\right)\nonumber\\
&=\frac{(\ell_i-\ell_j)}{\ell_i}\log\left(\sum_{c=j}^{A}\frac{\pi(c|s)}{\ell_c}\right)-\frac{(\ell_i-\ell_{j})}{\ell_i}\log\left(\sum_{c=j}^{A}\frac{\pi(c|s)}{\ell_c}\right)\nonumber\\
&=0\label{eq:dif_dif_kls}
\end{align}
where $(a)$ is due to the fact that all terms in $\sum_{c=i}^{A}\frac{\pi(c|s)}{\ell_c}$ are non-negative and $log$ is a monotone increasing function. Combining \eqref{eq:dif_dif_kls} with \eqref{eq:diff_kls} we get:
\begin{align}
	\frac{\partial v^{\pi_2}(s_1)}{\partial\alpha}\bigg|_{\alpha=0}&>\widehat{q}^{\pi_2}(s,a_i)-\widehat{q}^{\pi_2}(s,a_j)>0
\end{align}
which completes the proof.

The proof for the Mixed Pareto case follows similarly. We assume that there is a $uc$-optimal policy $\pi_1$ that assigns non-zero probability to an action $a_k$ Mixed Pareto dominated by $a_i$ and $a_j$, $\pi_1(a_k|s_1)>0$. Since $a_k$ is assumed to be Mixed Pareto dominated, equations \eqref{eq:required_optimality} are satisfied. Similarly, as before, we define a new policy $\pi_2$ as:
\begin{align}
&\pi_2(a|s)=\begin{cases}
\pi_1(a|s)-\alpha\hspace{5mm}\textrm{if $(s,a)=(s_1,a_k)$}\\
\pi_1(a|s)+\alpha\frac{\left(\ell_k-\ell_j\right)\ell_i}{\left(\ell_i-\ell_j\right)\ell_k}\hspace{5mm}\textrm{if $(s,a)=(s_1,a_i)$}\\
\pi_1(a|s)+\alpha\frac{\left(\ell_i-\ell_k\right)\ell_j}{\left(\ell_i-\ell_j\right)\ell_k}\hspace{5mm}\textrm{if $(s,a)=(s_1,a_j)$}\\
\pi_1(a|s)\hspace{10mm}\textrm{else}
\end{cases}
\end{align}
The gradient of the value function becomes:
\begin{align}
\frac{\partial v^{\pi_2}(s_1)}{\partial\alpha}\bigg|_{\alpha=0}&=\frac{\partial v^{\pi_2}(s_1)}{\partial\pi_1(a_k|s_1)}\frac{\partial \pi_1(a_k|s_1)}{\partial\alpha}\bigg|_{\alpha=0}+\frac{\partial v^{\pi_2}(s_1)}{\partial\pi_1(a_j|s_1)}\frac{\partial \pi_1(a_j|s_1)}{\partial\alpha}\bigg|_{\alpha=0}\nonumber\\
&\hspace{3mm}+\frac{\partial v^{\pi_2}(s_1)}{\partial\pi_1(a_i|s_1)}\frac{\partial \pi_1(a_i|s_1)}{\partial\alpha}\bigg|_{\alpha=0}\nonumber\\
&=\frac{\partial v^{\pi_2}(s_1)}{\partial\pi_1(a_i|s_1)}\frac{\left(\ell_k-\ell_j\right)\ell_i}{\left(\ell_i-\ell_j\right)\ell_k}+\frac{\partial v^{\pi_2}(s_1)}{\partial\pi_1(a_j|s_1)}\frac{\left(\ell_i-\ell_k\right)\ell_j}{\left(\ell_i-\ell_j\right)\ell_k}-\frac{\partial v^{\pi_2}(s_1)}{\partial\pi_1(a_k|s_1)}\nonumber\\
&=\underbrace{\widehat{q}^{\pi_2}(s,a_i)\frac{\left(\ell_k-\ell_j\right)\ell_i}{\left(\ell_i-\ell_j\right)\ell_k}+\widehat{q}^{\pi_2}(s,a_j)\frac{\left(\ell_i-\ell_k\right)\ell_j}{\left(\ell_i-\ell_j\right)\ell_k}-\widehat{q}^{\pi_2}(s,a_k)}_{>0\textrm{ (due to Pareto assumption)}}\nonumber\\
&+\kappa\left(\frac{\partial D_{KL}(u_{s_1}^{\pi_2}(\widetilde{q}))}{\partial\pi_1(a_k|s_1)}-\frac{\partial D_{KL}(u_{s_1}^{\pi_2}(\widetilde{q}))}{\partial\pi_1(a_i|s_1)}\frac{\left(\ell_k-\ell_j\right)\ell_i}{\left(\ell_i-\ell_j\right)\ell_k}-\frac{\partial D_{KL}(u_{s_1}^{\pi_2}(\widetilde{q}))}{\partial\pi_1(a_j|s_1)}\frac{\left(\ell_i-\ell_k\right)\ell_j}{\left(\ell_i-\ell_j\right)\ell_k}\right)\label{eq:diff_kls_mix}
\end{align}
And finally:
\begin{align}
	&\frac{\partial D_{KL}(u_{s_1}^{\pi_2}(\widetilde{q}))}{\partial\pi_1(a_k|s_1)}-\frac{\partial D_{KL}(u_{s_1}^{\pi_2}(\widetilde{q}))}{\partial\pi_1(a_i|s_1)}\frac{\left(\ell_k-\ell_j\right)\ell_i}{\left(\ell_i-\ell_j\right)\ell_k}-\frac{\partial D_{KL}(u_{s_1}^{\pi_2}(\widetilde{q}))}{\partial\pi_1(a_j|s_1)}\frac{\left(\ell_i-\ell_k\right)\ell_j}{\left(\ell_i-\ell_j\right)\ell_k}\nonumber\\
	&=\sum_{b=1}^k\frac{(\ell_b-\ell_{b-1})}{\ell_k}\log\left(\sum_{c=b}^{A}\frac{\pi(c|s)}{\ell_c}\right)-\frac{\left(\ell_k-\ell_j\right)}{\left(\ell_i-\ell_j\right)\ell_k}\sum_{b=1}^i(\ell_b-\ell_{b-1})\log\left(\sum_{c=b}^{A}\frac{\pi(c|s)}{\ell_c}\right)\nonumber\\
	&\hspace{25mm}-\frac{\left(\ell_i-\ell_k\right)}{\left(\ell_i-\ell_j\right)\ell_k}\sum_{b=1}^j(\ell_b-\ell_{b-1})\log\left(\sum_{c=b}^{A}\frac{\pi(c|s)}{\ell_c}\right)\nonumber\\
	&=\sum_{b=j+1}^k\frac{(\ell_b-\ell_{b-1})}{\ell_k}\log\left(\sum_{c=b}^{A}\frac{\pi(c|s)}{\ell_c}\right)-\frac{\left(\ell_k-\ell_j\right)}{\left(\ell_i-\ell_j\right)\ell_k}\sum_{b=j+1}^i(\ell_b-\ell_{b-1})\log\left(\sum_{c=b}^{A}\frac{\pi(c|s)}{\ell_c}\right)\nonumber\\
	&=\sum_{b=j+1}^k\frac{(\ell_b-\ell_{b-1})(\ell_i-\ell_k)}{(\ell_i-\ell_j)\ell_k}\log\left(\sum_{c=b}^{A}\frac{\pi(c|s)}{\ell_c}\right)-\frac{\left(\ell_k-\ell_j\right)}{\left(\ell_i-\ell_j\right)\ell_k}\sum_{b=k+1}^i(\ell_b-\ell_{b-1})\log\left(\sum_{c=b}^{A}\frac{\pi(c|s)}{\ell_c}\right)\nonumber\\
	&>\sum_{b=j+1}^k\frac{(\ell_b-\ell_{b-1})(\ell_i-\ell_k)}{(\ell_i-\ell_j)\ell_k}\log\left(\sum_{c=k}^{A}\frac{\pi(c|s)}{\ell_c}\right)-\frac{\left(\ell_k-\ell_j\right)}{\left(\ell_i-\ell_j\right)\ell_k}\sum_{b=k+1}^i(\ell_b-\ell_{b-1})\log\left(\sum_{c=k}^{A}\frac{\pi(c|s)}{\ell_c}\right)\nonumber\\
	&=\frac{(\ell_k-\ell_{j})(\ell_i-\ell_k)}{(\ell_i-\ell_j)\ell_k}\log\left(\sum_{c=k}^{A}\frac{\pi(c|s)}{\ell_c}\right)-\frac{\left(\ell_k-\ell_j\right)(\ell_i-\ell_{k})}{\left(\ell_i-\ell_j\right)\ell_k}\log\left(\sum_{c=k}^{A}\frac{\pi(c|s)}{\ell_c}\right)=0\label{eq:dif_dif_kls2}
\end{align}
 Combining \eqref{eq:diff_kls_mix} with \eqref{eq:dif_dif_kls2} we get $\frac{\partial v^{\pi_2}(s_1)}{\partial\alpha}\big|_{\alpha=0}>0$, which completes the proof.
 
\subsection{Proof Theorem \ref{lemma:opt_policy}} \label{appendix:theorem}
Due to Lemma \ref{lemma:Pareto_dominated} we already know that for Pareto dominated actions $\pi^\star(a|s)=0$. Therefore, without loss of generality, we assume that all actions are Pareto optimal. Furthermore, to simplify notation in the proof we assume that $\sigma_s(a_j)=j$ and hence we will not use the ordering function $\sigma_s(a)$. We start differentiating \eqref{eq:optimization_problem_nice} with respect to $\pi(a|s)$ and equating to zero:
\begin{align}\label{eq:diff_rel}
\kappa^{-1}\widehat{q}(s,a)-\frac{\partial D_{KL}(u_s^\star(\widetilde{q}))}{\partial\pi^\star(a|s)}=0
\end{align}
Using \eqref{eq:dkl} we get the following expression for the gradient of the KL term:
\begin{align}
\frac{\partial D_{KL}(u_s^\star(\widetilde{q}))}{\partial\pi^\star(a_j|s)}&=\sum_{b=1}^j\frac{(\ell_b-\ell_{b-1})}{\ell_j}\log\left(\sum_{c=b}^{A}\frac{\pi^\star(c|s)}{\ell_c}\right)+\ell_j^{-1}\sum_{b=1}^j(\ell_b-\ell_{b-1})\nonumber\\
&=\sum_{b=1}^j\frac{(\ell_b-\ell_{b-1})}{\ell_j}\log\left(\sum_{c=b}^{A}\frac{\pi^\star(c|s)}{\ell_c}\right)+1\nonumber\\
&=\ell_j^{-1}(\ell_j-\ell_{j-1})\log\left(\sum_{c=j}^{A}\pi^\star(c|s)\ell_c^{-1}\right)+\frac{\ell_{j-1}}{\ell_{j}}\left(\frac{\partial D_{KL}(u_s^\star(\widetilde{q}))}{\partial\pi(a_{j-1}|s)}-1\right)+1\nonumber\\
&=\ell_j^{-1}(\ell_j-\ell_{j-1})\log\left(\sum_{c=j}^{A}\pi^\star(c|s)\ell_c^{-1}\right)+\frac{\ell_{j-1}}{\ell_{j}}\frac{\partial D_{KL}(u_s^\star(\widetilde{q}))}{\partial\pi^\star(a_{j-1}|s)}+\frac{\ell_{j}-\ell_{j-1}}{\ell_{j}}\label{eq:rec_gradient}
\end{align}
Now we can solve for each action combining the recursive form given in \eqref{eq:rec_gradient} with \eqref{eq:diff_rel}. Recall that due to the specific numbering of actions we assumed, $a_A$ is the action who has the greatest uncertainty $\ell_A$. Hence, we can start solving for $a_A$ as follows:
\begin{align}
\allowdisplaybreaks
0&=\kappa^{-1}\widehat{q}(s,a_A)-\ell_A^{-1}(\ell_{A}-\ell_{A-1})\log\left(\pi^\star(a_A|s)\ell_{A}^{-1}\right)-\frac{\ell_{A-1}}{\ell_{A}}\frac{\partial D_{KL}(u_s^\star(\widetilde{q})||u_s^\bullet(\widetilde{q}))}{\partial\pi^\star(a_{A-1}|s)}-\frac{\ell_{A}-\ell_{A-1}}{\ell_{A}}\nonumber\\
&=\kappa^{-1}\widehat{q}(s,a_A)-\ell_A^{-1}(\ell_{A}-\ell_{A-1})\log\left(\pi^\star(a_A|s)\ell_{A}^{-1}\right)-\kappa^{-1}\frac{\ell_{A-1}}{\ell_{A}}\widehat{q}(s,a_{A-1})-\frac{\ell_{A}-\ell_{A-1}}{\ell_{A}}\nonumber\\
&\rightarrow\pi^\star(a_A|s)\propto \ell_Ap_A(s)\label{eq:pi_A}
\end{align}
where we defined:
\begin{align}\label{eq:p}
p_j(s)=\mathrm{exp}\left[\frac{\ell_{j}(s)\widehat{q}(s,a_j)-\ell_{j-1}(s)\widehat{q}(s,a_{j-1})}{\kappa\left(\ell_{j}(s)-\ell_{j-1}(s)\right)}\hspace{-0.4mm}\right]
\end{align}
Following the same procedure as in \eqref{eq:pi_A} we can solve for $\pi^\star(a_j|s)$.
\begin{align}
0&=\kappa^{-1}\widehat{q}(j,s)-\frac{(\ell_{j}-\ell_{j-1})}{\ell_j}\log\left(\sum_{c=j}^{A}\frac{\pi^\star(c|s)}{\ell_c}\right)-\frac{\ell_{j-1}}{\ell_{j}}\frac{\partial D_{KL}(u_s^\star(\widetilde{q}))}{\partial\pi^\star(a_{j-1}|s)}-\frac{\ell_{j}-\ell_{j-1}}{\ell_{j}}\nonumber\\
&\stackrel{(a)}{=}\kappa^{-1}\widehat{q}(s,a_j)-\frac{(\ell_{j}-\ell_{j-1})}{\ell_j}\log\left(\sum_{c=j}^{A}\frac{\pi^\star(c|s)}{\ell_c}\right)-\kappa^{-1}\frac{\ell_{j-1}}{\ell_{j}}\widehat{q}(s,a_{j-1})-\frac{\ell_{j}-\ell_{j-1}}{\ell_{j}}\nonumber\\
&\rightarrow\pi^\star(a_j|s)\propto\ell_jp_j(s)e^{-1}-\ell_j\sum_{c=j+1}^{A}\pi^\star(c|s)\ell_{c}^{-1}\label{eq:policy_rec}
\end{align}
where in $(a)$ we used \eqref{eq:diff_rel}. Starting with $j=A-1$, unwinding \eqref{eq:policy_rec} one step of the recursion at a time, and normalizing we get:
\begin{align}
\pi^\star(a_j|s)&=\frac{\ell_j(p_j(s)-p_{j+1}(s))}{\sum_{j=1}^A(\ell_j-\ell_{j-1})p_j(s)}
\end{align}
which completes the proof.

\subsection{Proof Lemma \ref{lemma:policy_eval}}
\label{appendix:lemma_policy_eval}
We start by showing that $\mathcal{T^\ell}$ is a contraction mapping. For this we define two mappings $q_1,q_2:\mathcal{S}\times\mathcal{A}:\rightarrow\mathbb{R}$. We start assuming $\mathcal{T^\ell}q_1(s,a)\geq\mathcal{T^\ell}q_2(s,a)$, then:
\begin{align}
0&\leq\mathcal{T^\ell}q_1(s,a)-\mathcal{T^\ell}q_2(s,a)=r(s,a)+\gamma\Ex_{\bs{s'}}v_1^\star(\bs{s'})-r(s,a)-\gamma\Ex_{\bs{s'}}v_2^\star(\bs{s'})=\gamma\Ex_{\bs{s'}}\left(v_1^\star(\bs{s'})-v_2^\star(\bs{s'})\right)\nonumber\\
&\stackrel{(a)}{=}\gamma\Ex_{\bs{s'}}\left(\max_{\pi}\Ex_{\bs{a}\sim\pi}\left(q_1(\bs{s'},\bs{a})-\kappa D_{KL}(u_{\bs{s'}}^{\pi}(\widetilde{q}))\right)-\max_{\pi}\Ex_{\bs{a}\sim\pi}\left(q_2(\bs{s'},\bs{a})-\kappa D_{KL}(u_{\bs{s'}}^{\pi}(\widetilde{q}))\right)\right)\nonumber\\
&\leq\gamma\Ex_{\bs{s'}}\max_{\pi}\Ex_{\bs{a}\sim\pi}\left(q_1(\bs{s'},\bs{a})-\kappa D_{KL}(u_{\bs{s'}}^{\pi}(\widetilde{q}))-q_2(\bs{s'},\bs{a})+\kappa D_{KL}(u_{\bs{s'}}^{\pi}(\widetilde{q}))\right)\nonumber\\
&=\gamma\Ex_{\bs{s'}}\max_{\pi}\Ex_{\bs{a}\sim\pi}\left(q_1(\bs{s'},\bs{a})-q_2(\bs{s'},\bs{a})\right)\nonumber\\
&\leq\gamma\max_{s,a}\left(q_1(s,a)-q_2(s,a)\right)\\
&\rightarrow 0\leq\mathcal{T^\ell}q_1(s,a)-\mathcal{T^\ell}q_2(s,a)\leq\gamma\max_{s,a}\left(q_1(s,a)-q_2(s,a)\right)
\end{align}
where in $(a)$ we used $v^\star(s)=\max_{\pi}\Ex_{\bs{a}\sim\pi}\left(q^\star(s,\bs{a})-\kappa D_{KL}(u_s^\pi(\widetilde{q}))\right)$. Noting that if $\mathcal{T^\ell}q_1(s,a)<\mathcal{T^\ell}q_2(s,a)$ the same argument applies. Exchanging the roles of $q_1$ and $q_2$ we can conclude that for any $(s,a)$ pair it holds:
\begin{align}
&0\leq|\mathcal{T^\ell}q_1(s,a)-\mathcal{T^\ell}q_2(s,a)|\leq\gamma\max_{s,a}\left|q_1(s,a)-q_2(s,a)\right|
\end{align}
which concludes the proof that $\mathcal{T^\ell}$ is a contraction mapping. Applying Banach's Fixed-Point Theorem concludes the proof (see Theorem 6.2.3 \cite{puterman}).

\subsection{Proof Lemma \ref{lemma:l_eval}}
\label{appendix:l_eval}
The proof follows by noting that due to Lemma 4 after application of $\ell$-Policy Evaluation, it will hold $\delta(s,a)=0$ for any $(s,a)$ pair. Therefore,
\begin{align}
\ell(s,a)\leftarrow|\delta(s,a)|+\gamma\Ex_{\bm{s}'}\max_{a}\ell(\bm{s}',a)\hspace{0.6mm}=\hspace{0.6mm}\gamma\Ex_{\bm{s}'}\max_{a}\ell(\bm{s}',a)
\end{align}
and hence $\ell(s,a)$ for any $(s,a)$ pair decays $\gamma$-linearly to 0. Combining this result with Remark \ref{remark:limit} concludes the proof.

\subsection{Cartpole Swingup Implementation Details}\label{appendix:cartpole}
The implementation details are as follows. All implementations used TensorFlow. We used neural networks as function approximators in all cases. All NN's are composed of two hidden layers with fifty units per layer. ReLu's are used in all hidden layers. All output layers are linear, except for the outputs of the networks that approximate the $\ell$ values whose outputs pass through sigmoid functions with limits $[1e-12,100]$. We used the ADAM optimizer in all cases. To approximate the $\ell$ values we used one network with only one output per action instead of one network with $A$ outputs, empirically this provides better performance without making any difference in terms of computation requirements. All hyperparameters were set by iterating through them and performing individual per-hyperparameter grid-searches; resulting values are shown in table \ref{table:Hyperparameters_cartpole}.
\begin{table}[H]
	\caption{Hyperparameters for Cartpole Swingup.}
	\label{table:Hyperparameters_cartpole}
	\begin{center}
		\begin{small}
			\begin{tabular}{ccccr}
				\toprule
				ISL & BDQN & UBE & SBEED \\
				\midrule
				$\gamma=0.99$ & $\gamma=0.99$ & $\gamma=0.99$ & $\gamma=0.99$\\
				replay buffer size = $1e5$ & replay buffer size = $1e5$ & replay buffer size = $1e5$ & replay buffer size = $1e5$\\
				batch size $=64$    & batch size $= 128$ & batch size $=128$ & batch size $=256$ \\
				learning rate ($q$) = $2e-4$  & learning rate ($q$) = $5e-4$ & learning rate ($q$) = $5e-4$ & learning rate ($\rho$) = $1e-3$ \\
				learning rate ($\rho$) = $5e-6$      & $\epsilon=0$ & learning rate (u) = $1e-4$ & learning rate ($v$) = $1e-3$ \\
				learning rate ($\ell$) = $2e-5$      & mask probability $=0.5$ & $\mu$ = $20$  &learning rate ($\pi$) = $1e-3$ \\
				target update period $= 4$     & target update period $= 4$ & target update period $= 4$ & $\kappa=0.5$\\
				$T=1$    &  sgd period$=1$ & $T=1$ & $T=1$\\
				$I=3$    &      & $I=1$  & $I=1$\\
				$\eta_1=0.8$      & & $\beta$=$1$ & $\eta=0$\\
				$\eta_2=0.7$      &  \\
				$\kappa=13$   &  \\
				\bottomrule
			\end{tabular}
		\end{small}
	\end{center}
\end{table}

\subsection{Deep Sea Implementation Details}\label{appendix:deep_sea}
The architecture of the implementation is the same as the one used for the Cartpole Swingup task. All hyperparameters were set by iterating through them and performing individual per-hyperparameter grid-searches; resulting values are shown in table \ref{table:Hyperparameters_deep_sea}.
\begin{table}[H]
	\caption{Hyperparameters for Deep Sea.}
	\label{table:Hyperparameters_deep_sea}
	%\vskip -0.15in
	\begin{center}
		\begin{small}
			\begin{tabular}{ccccr}
				\toprule
				ISL & BDQN & UBE & SBEED \\
				\midrule
				$\gamma=0.99$ & $\gamma=0.99$ & $\gamma=0.99$ & $\gamma=0.99$\\
				replay buffer size = $1e5$ & replay buffer size = $1e5$ & replay buffer size = $1e5$ & replay buffer size = $1e5$\\
				batch size $=256$    & batch size $= 128$ & batch size $=128$ & batch size $=256$ \\
				learning rate ($q$) = $2e-4$      & learning rate ($q$) = $5e-4$ & learning rate ($q$) = $5e-4$ & learning rate ($\rho$) = $1e-2$ \\
				learning rate ($\rho$) = $1e-4$      & $\epsilon=0$ & learning rate (u) = $1e-4$  & learning rate ($v$) = $1e-2$ \\
				learning rate ($\ell$) = $5e-5$      & mask probability $=0.5$ & $\mu$ = $10$  & learning rate ($\pi$) = $1e-2$ \\
				target update period $=2$     & target update period $= 5$ & target update period $= 4$ & $\kappa=0.5$\\
				$T=2$    &  sgd period$=1$ & $T=1$& $T=1$\\
				$I=1$    &         & $I=1$& $I=1$\\
				$\eta_1=0.9$      &  & $\beta=2$ & $\eta=0$\\
				$\eta_2=0.1$      &  \\
				$\kappa=1$   &  \\
				\bottomrule
			\end{tabular}
		\end{small}
	\end{center}
\end{table}

\subsection{Deep Sea Stochastic Implementation Details}\label{appendix:deep_sea_sto}
The implementation architecture for the Deep Sea Stochastic game is the same as for Deep Sea, only some hyperparameters change (see table \ref{table:Hyperparameters_deep_sea_sto}).
\begin{table}[H]
	\caption{Hyperparameters for Deep Sea Stochastic.}
	\label{table:Hyperparameters_deep_sea_sto}
	%\vskip -0.15in
	\begin{center}
		\begin{small}
			\begin{tabular}{ccccr}
				\toprule
				ISL & BDQN & UBE & SBEED \\
				\midrule
				$\gamma=0.99$ & $\gamma=0.99$ & $\gamma=0.99$ & $\gamma=0.99$ \\
				replay buffer size = $1e5$ & replay buffer size = $1e5$ & replay buffer size = $1e5$ & replay buffer size = $1e5$ \\
				batch size $=256$    & batch size $= 128$ & batch size $= 128$ & batch size $=256$ \\
				learning rate ($q$) = $1e-4$      & learning rate ($q$) = $1e-5$ & learning rate ($q$) = $2e-5$ & learning rate ($\rho$) = $1e-3$ \\
				learning rate ($\rho$) = $1e-4$      & $\epsilon=0$ & learning rate (u) = $1e-5$ & learning rate ($v$) = $1e-3$ \\
				learning rate ($\ell$) = $5e-5$      & mask probability $=0.5$ &$\mu$=$10$& learning rate ($\pi$) = $1e-3$ \\
				target update period $=2$     & target update period $= 5$ & target update period $= 4$ & $\kappa=0.5$\\
				$T=10$    &  sgd period$=1$ & $T=1$& $T=1$\\
				$I=1$    &         & $I=1$& $I=1$\\
				$\eta_1=1.0$      &  &$\beta=2$& $\eta=0.01$\\
				$\eta_2=0.5$      &  \\
				$\kappa=1$   &  \\
				\bottomrule
			\end{tabular}
		\end{small}
	\end{center}
\end{table}

\subsection{Ablation Study}\label{appendix:ablation}
In this section we include an ablation study for the hyperparameters $\kappa$, $\eta_1$ and $\eta_2$ using the cartpole task. We sweep $\kappa$ through $20\times2^{[-4,-3,-2,-1,0,1,2,3,4]}$ and $\eta_1$ and $\eta_2$ through $[0,0.1,\cdots,0.9,1]$. Results are shown in figure \ref{fig:ablation_figures}; to make these plots more visually clear we only show three of the curves in each plot; we note though that the effects of the hyperparameters on the performance of ISL are clear with these three curves. The rest of the curves can be found figures \ref{fig:ablation_k_all}, \ref{fig:ablation_eta1_all} and \ref{fig:ablation_eta2_all}. The exploration-exploitation trade-off managed by $\kappa$ is clear in figure \ref{fig:ablation_k}; increasing $\kappa$ improves results for high values of $N$ (since more exploration is required in these cases) but does so at the expense of exploiting less and therefore the best return diminishes for low values of $N$ (where less exploration is necessary). Figures \ref{fig:ablation_e1} and \ref{fig:ablation_e2} indicate that the effect of hyperparameters $\eta_1$ and $\eta_2$ on the performance of ISL are less drastic than that of $\kappa$, however in both cases the best performance is obtained for intermediate values of $\eta_1$ and $\eta_2$ as expected.
\begin{figure}[H]
	%\vskip 0.2in
	\begin{center}
		\centering
		\begin{subfigure}{0.32\textwidth}
			\includegraphics[width=\textwidth]{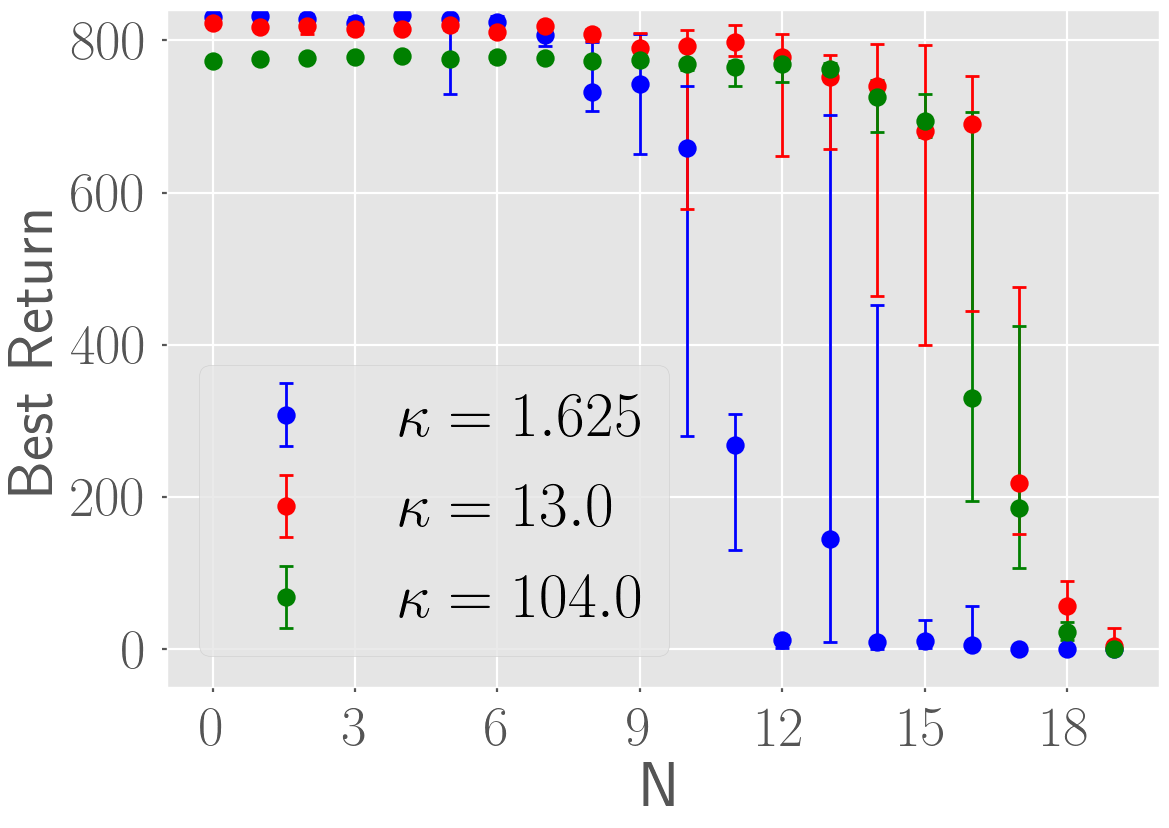}
			\subcaption{$\kappa$}
			\label{fig:ablation_k}
		\end{subfigure}
		\begin{subfigure}{0.32\textwidth}
			\includegraphics[width=\textwidth]{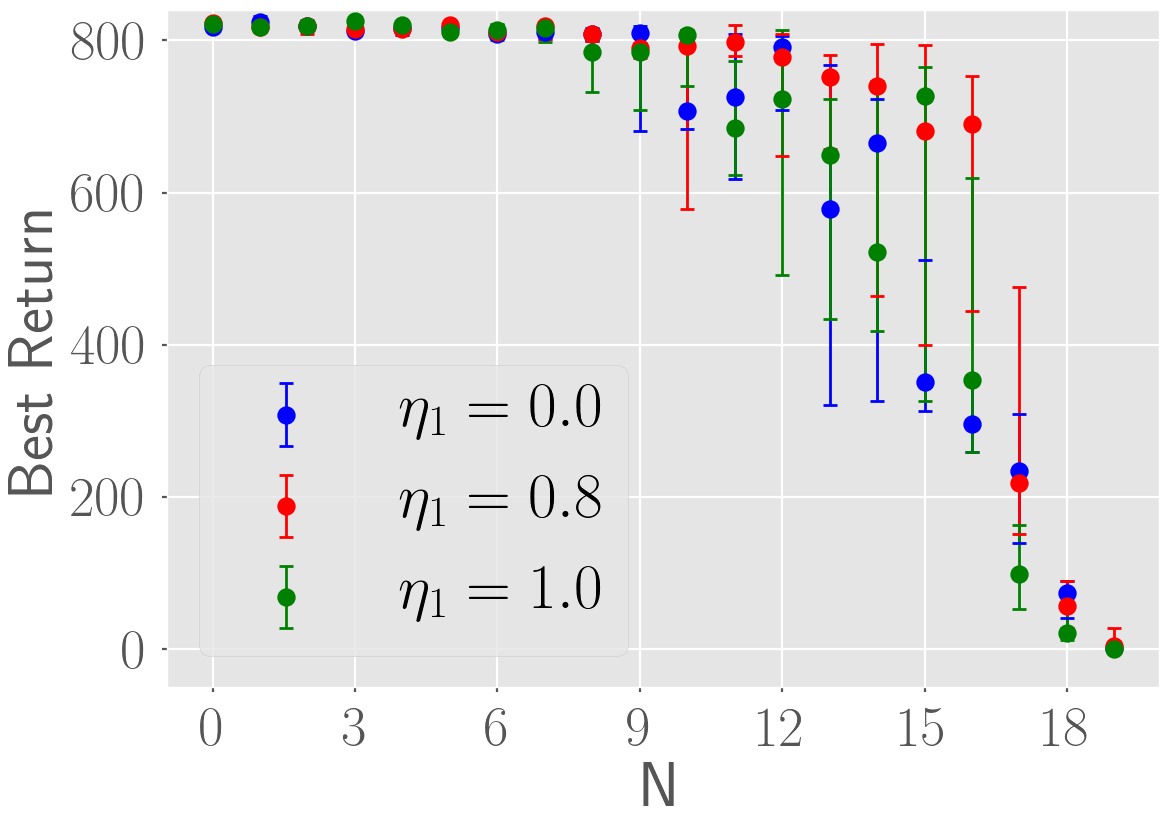}
			\subcaption{$\eta_1$}
			\label{fig:ablation_e1}
		\end{subfigure}
		\begin{subfigure}{0.32\textwidth}
			\includegraphics[width=\textwidth]{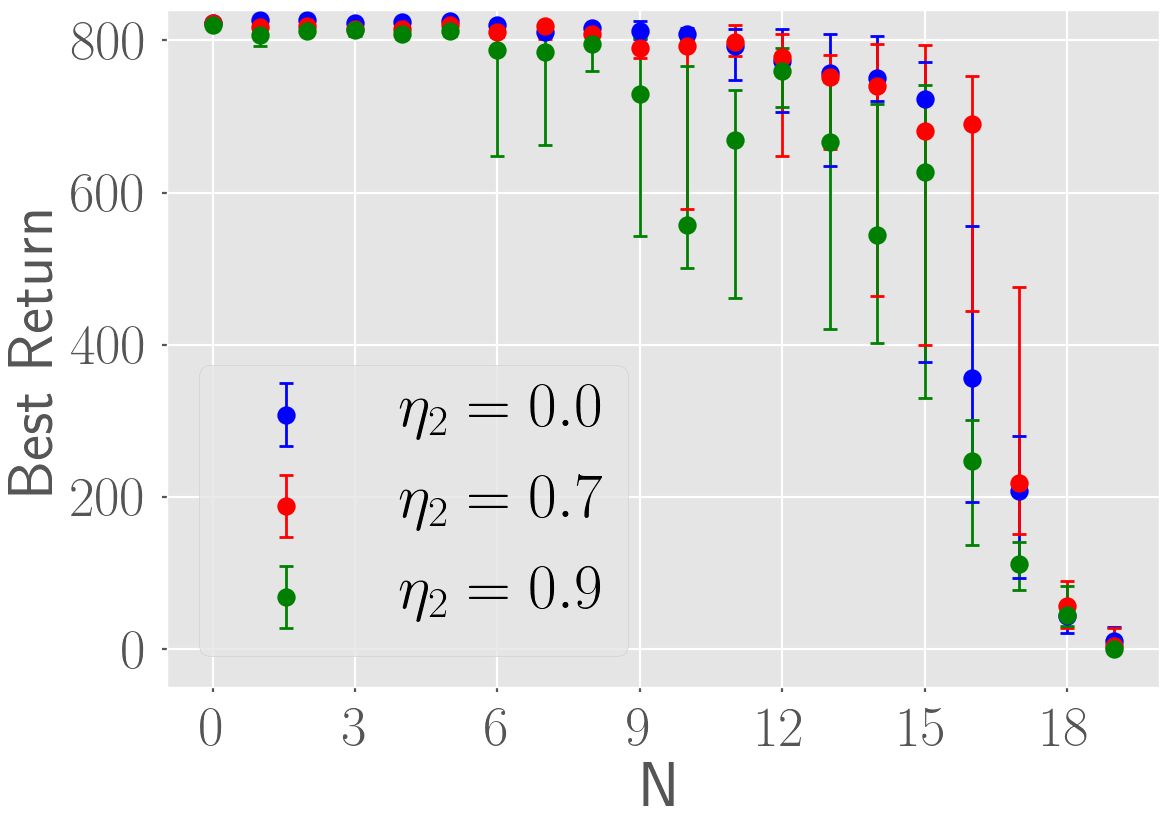}
			\subcaption{$\eta_2$}
			\label{fig:ablation_e2}
		\end{subfigure}
		\caption{In all cases we ran 10 experiments with different seeds, the plots show the median and first and third quartiles.}
		\label{fig:ablation_figures}
	\end{center}
	%\vskip -0.2in
\end{figure}

\subsubsection{Ablation study for $\kappa$}

\begin{figure}[H]
	\centering
	\begin{subfigure}{.3\textwidth}
		\includegraphics[width=\textwidth]{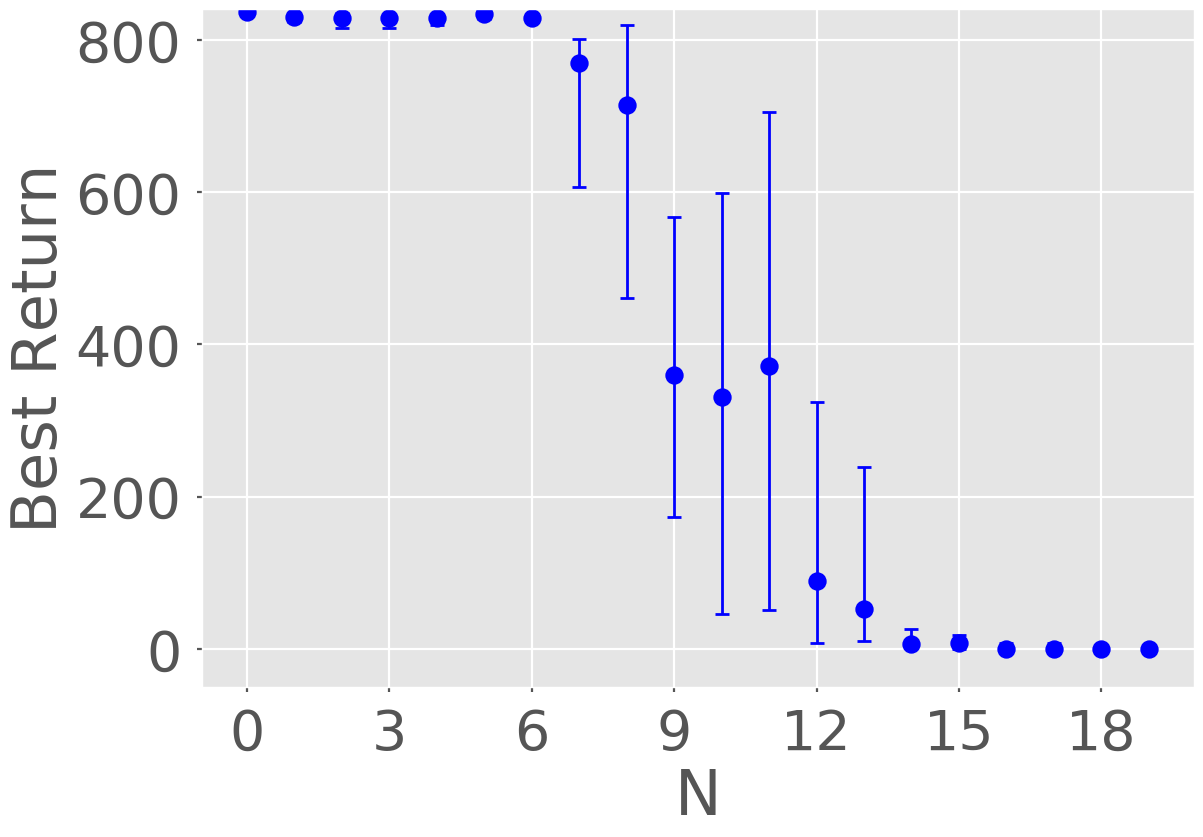}
		\subcaption{$\kappa=0.8125$}
	\end{subfigure}
	\begin{subfigure}{.3\textwidth}
		\includegraphics[width=\textwidth]{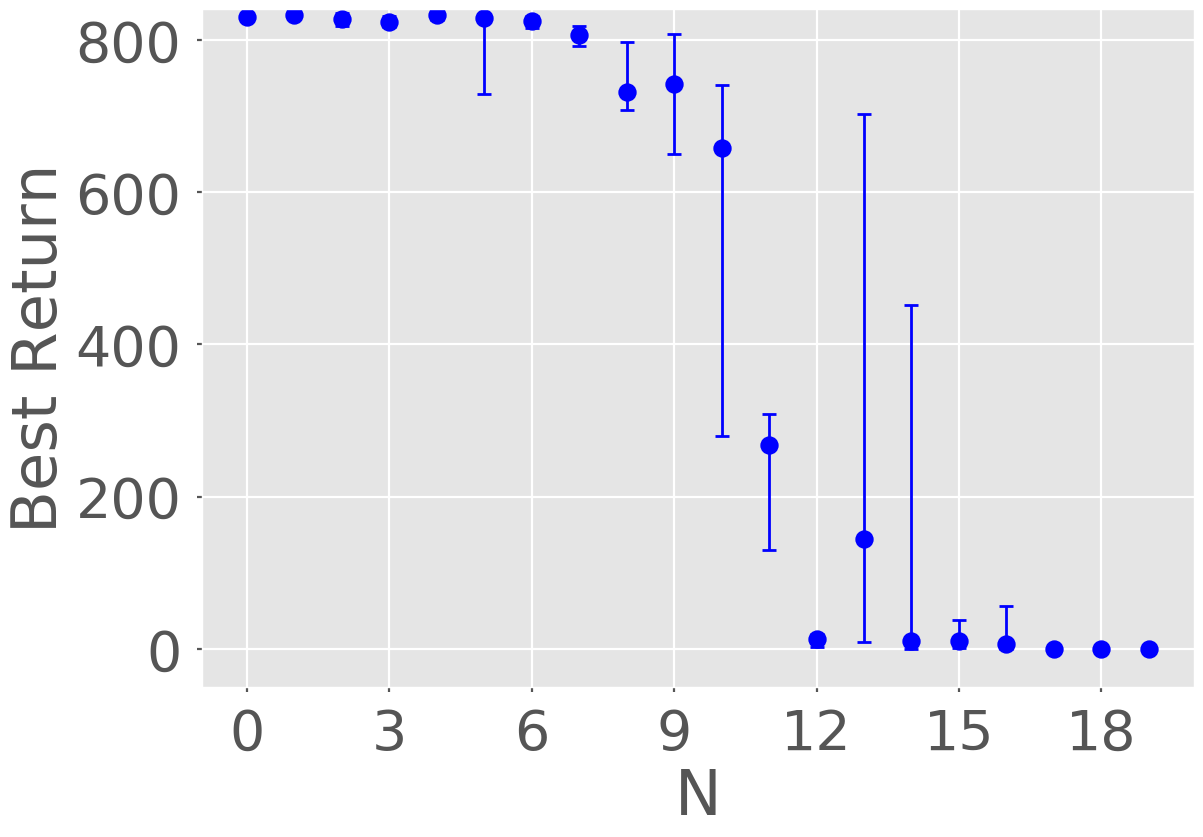}
		\subcaption{$\kappa=1.625$}
	\end{subfigure}
	\begin{subfigure}{.3\textwidth}
		\includegraphics[width=\textwidth]{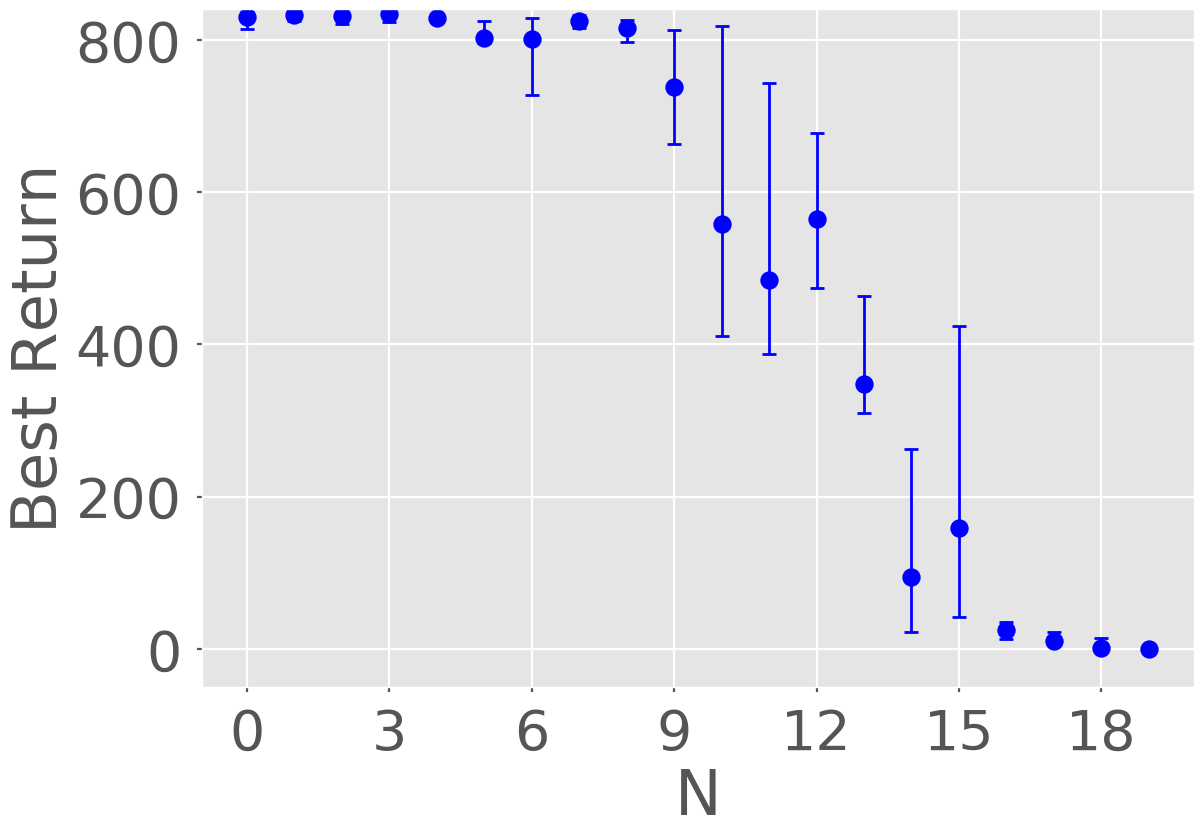}
		\subcaption{$\kappa=3.25$}
	\end{subfigure}
	\begin{subfigure}{.3\textwidth}
		\includegraphics[width=\textwidth]{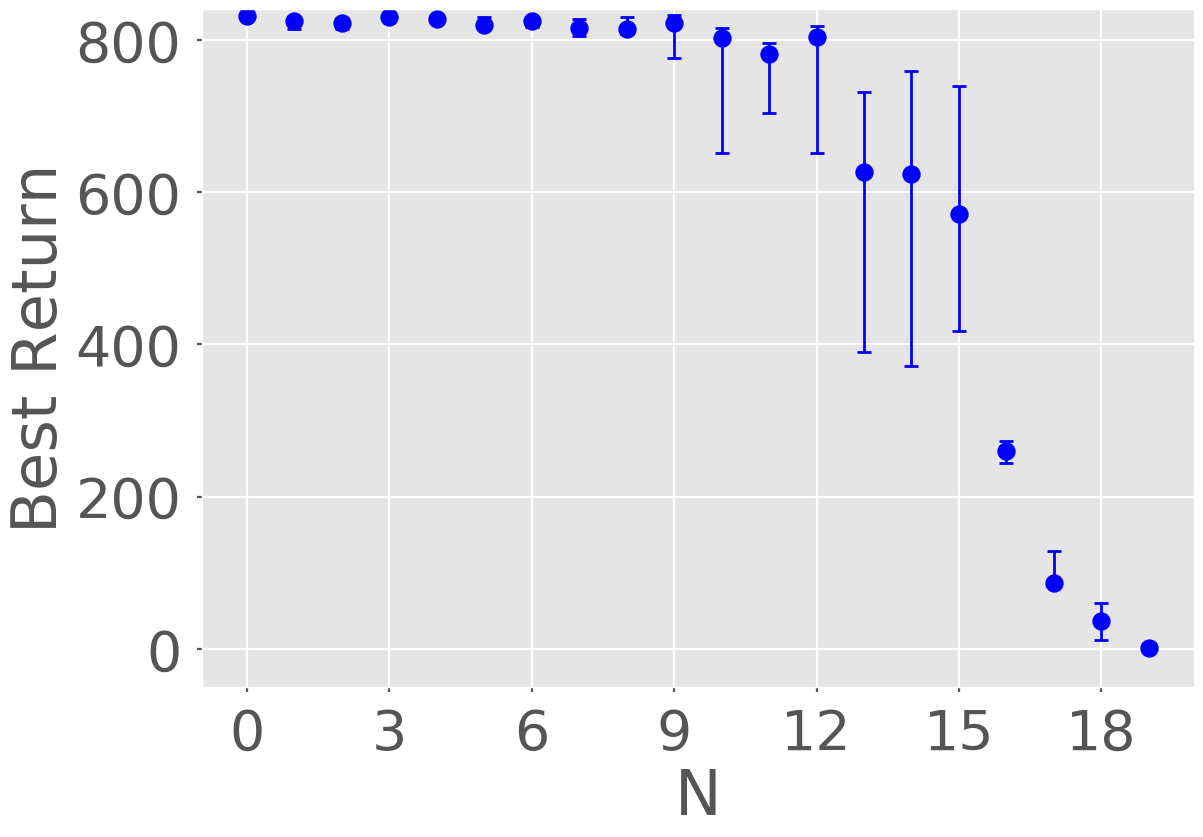}
		\subcaption{$\kappa=6.5$}
	\end{subfigure}
	\begin{subfigure}{.3\textwidth}
		\includegraphics[width=\textwidth]{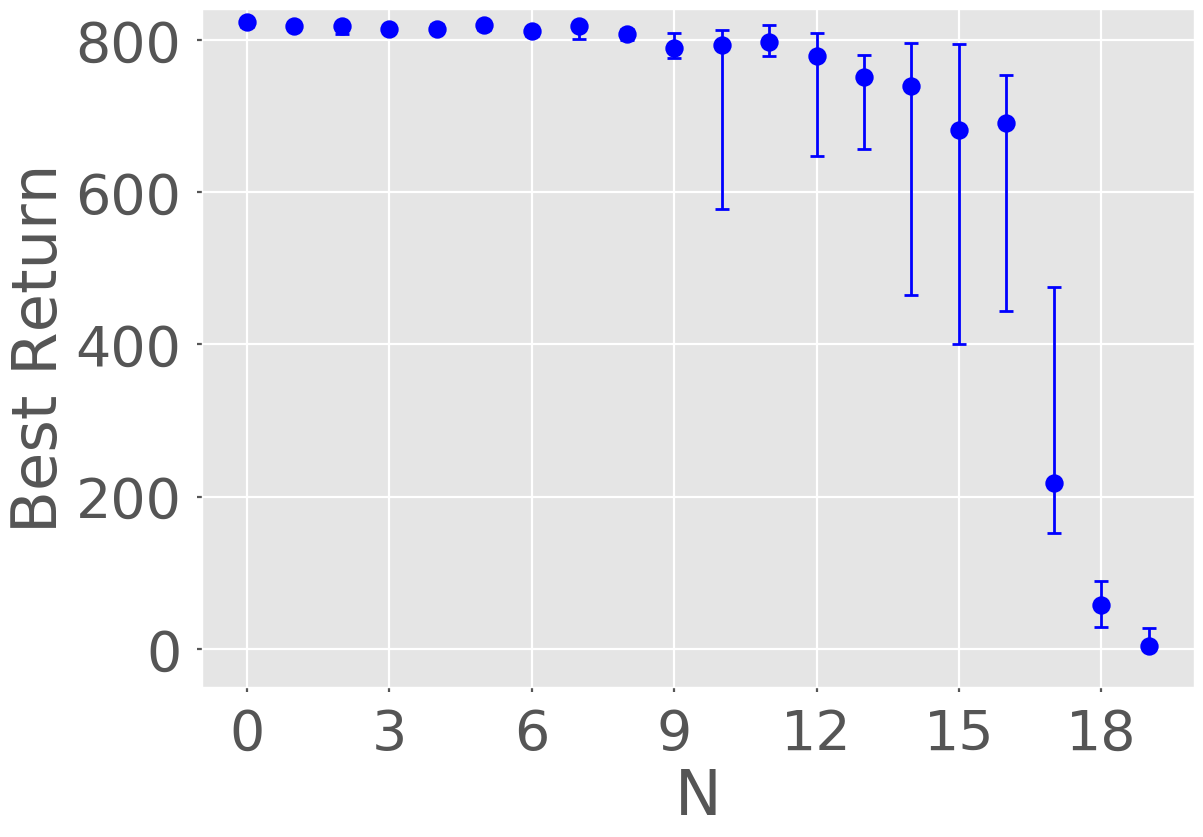}
		\subcaption{$\kappa=13$}
	\end{subfigure}
	\begin{subfigure}{.3\textwidth}
		\includegraphics[width=\textwidth]{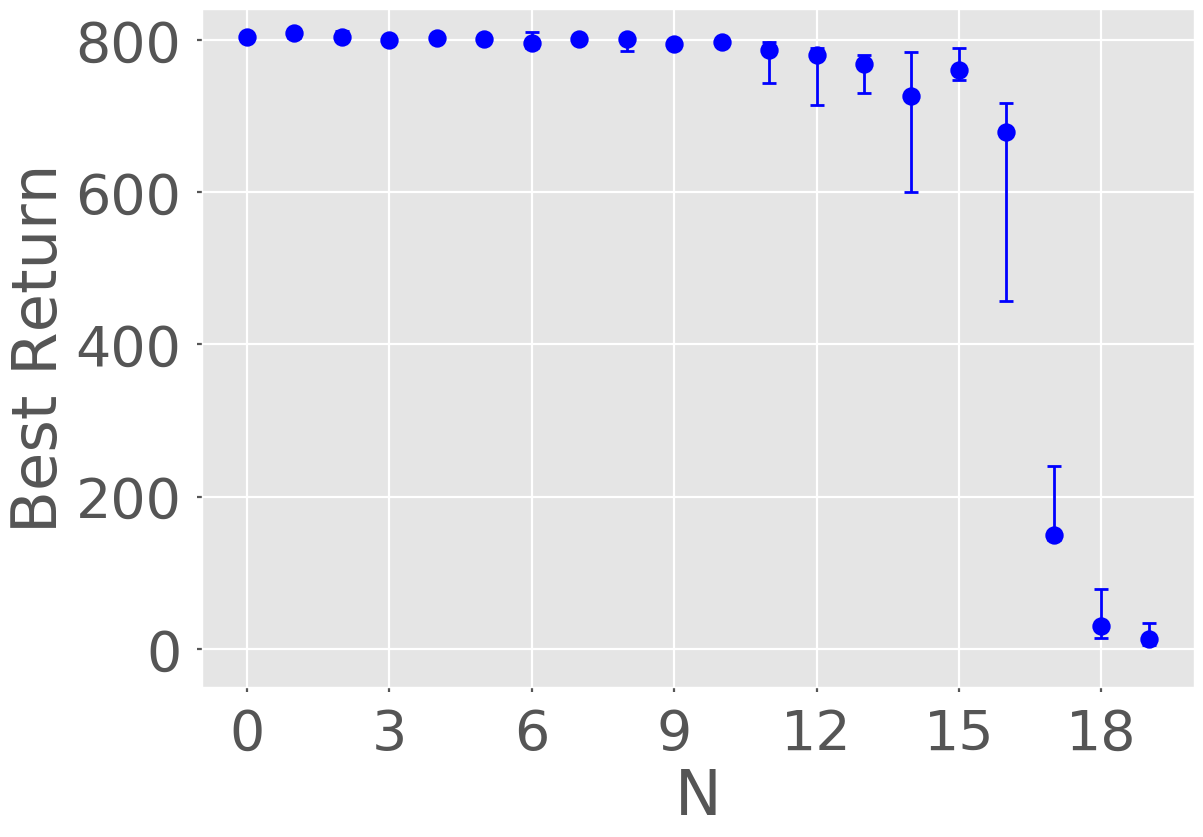}
		\subcaption{$\kappa=26$}
	\end{subfigure}
	\begin{subfigure}{.3\textwidth}
		\includegraphics[width=\textwidth]{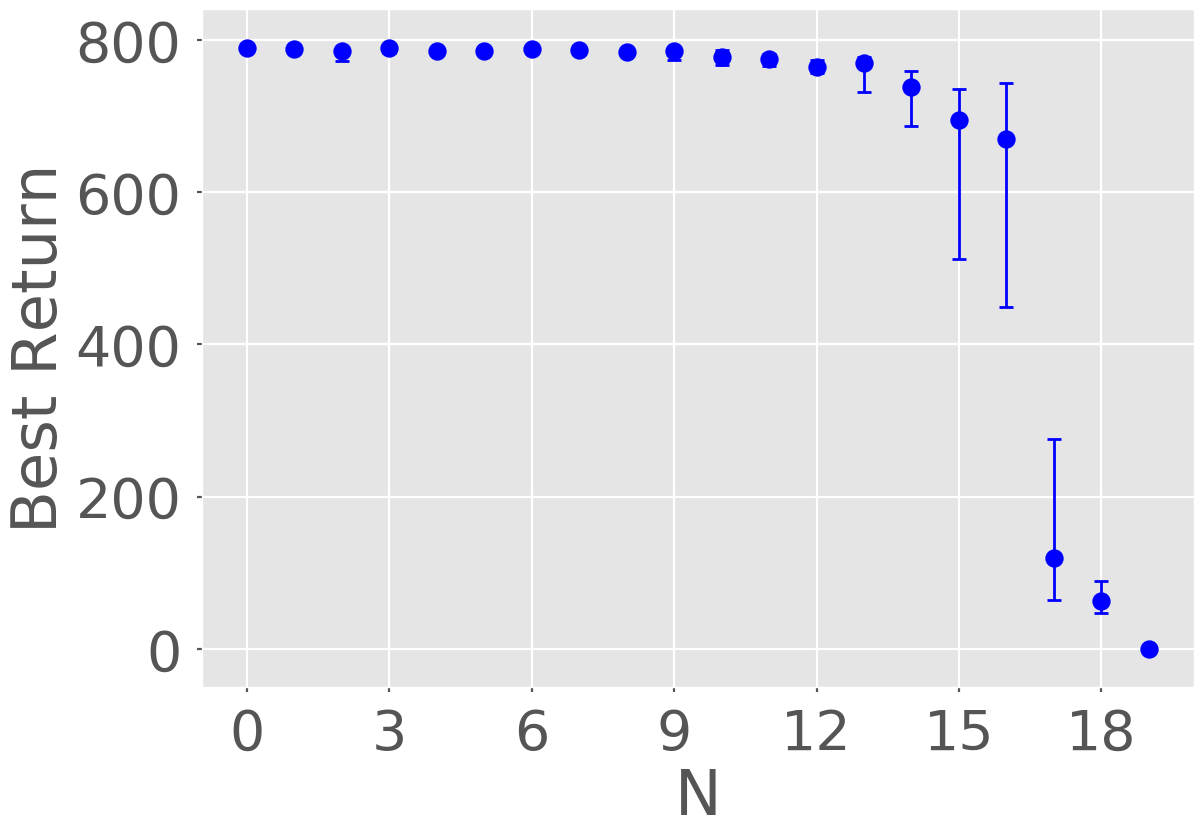}
		\subcaption{$\kappa=52$}
	\end{subfigure}
	\begin{subfigure}{.3\textwidth}
		\includegraphics[width=\textwidth]{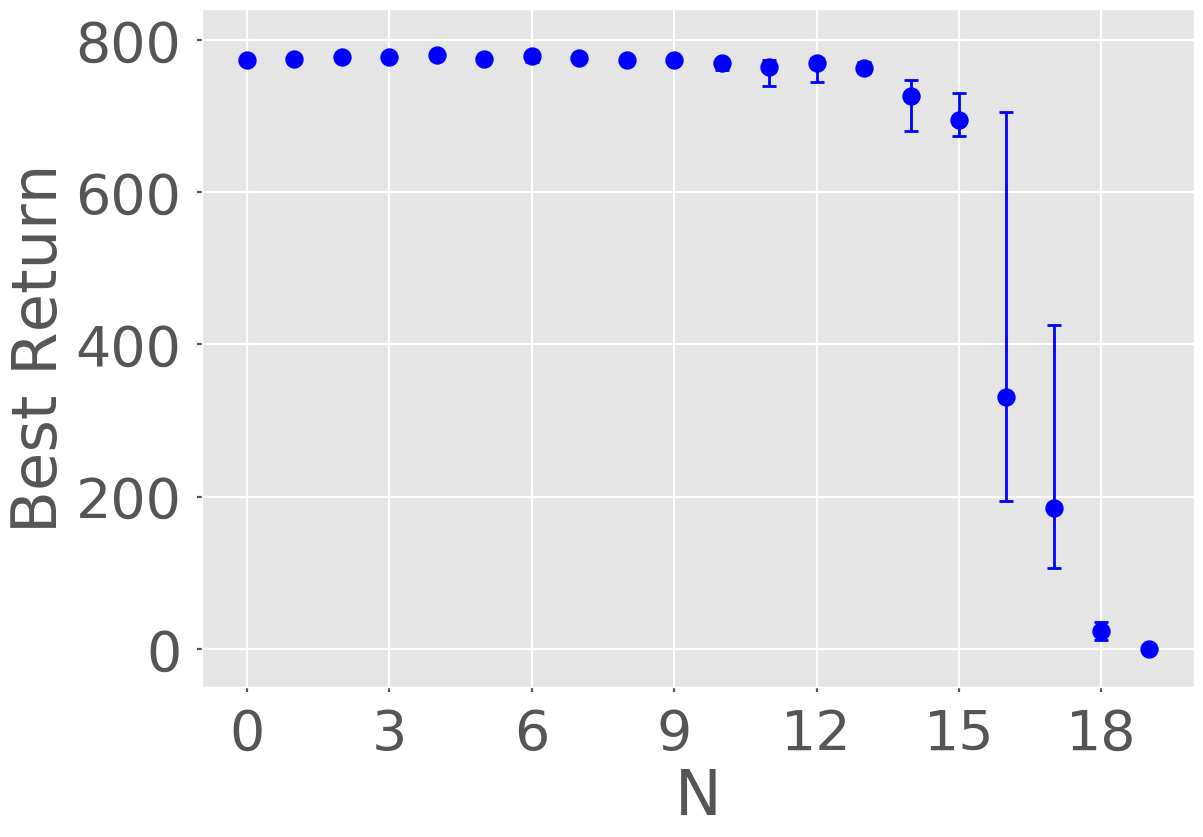}
		\subcaption{$\kappa=104$}
	\end{subfigure}
	\begin{subfigure}{.3\textwidth}
		\includegraphics[width=\textwidth]{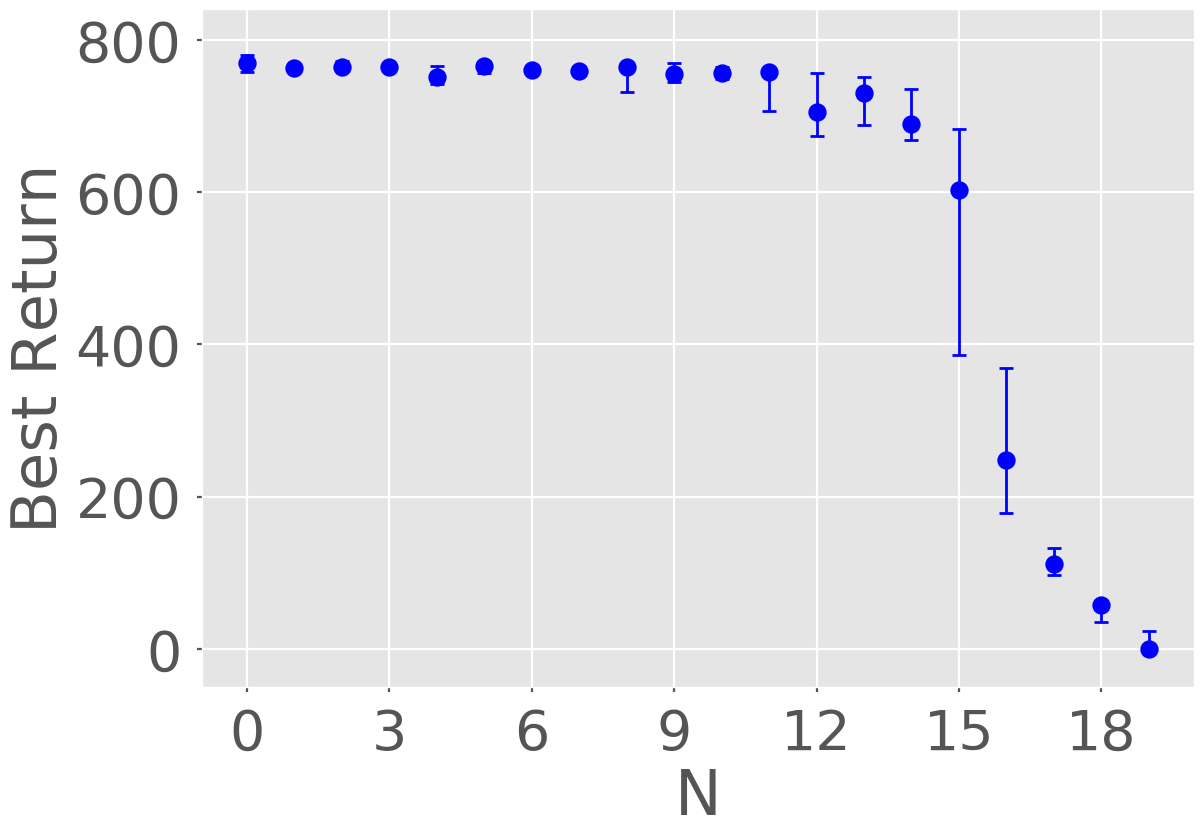}
		\subcaption{$\kappa=208$}
	\end{subfigure}
	\caption{Ablation study for $\kappa$.}
	\label{fig:ablation_k_all}
\end{figure}
\newpage

\subsubsection{Ablation study for $\eta_1$}
\begin{figure}[H]
	\centering
	\begin{subfigure}{.32\textwidth}
		\includegraphics[width=\textwidth]{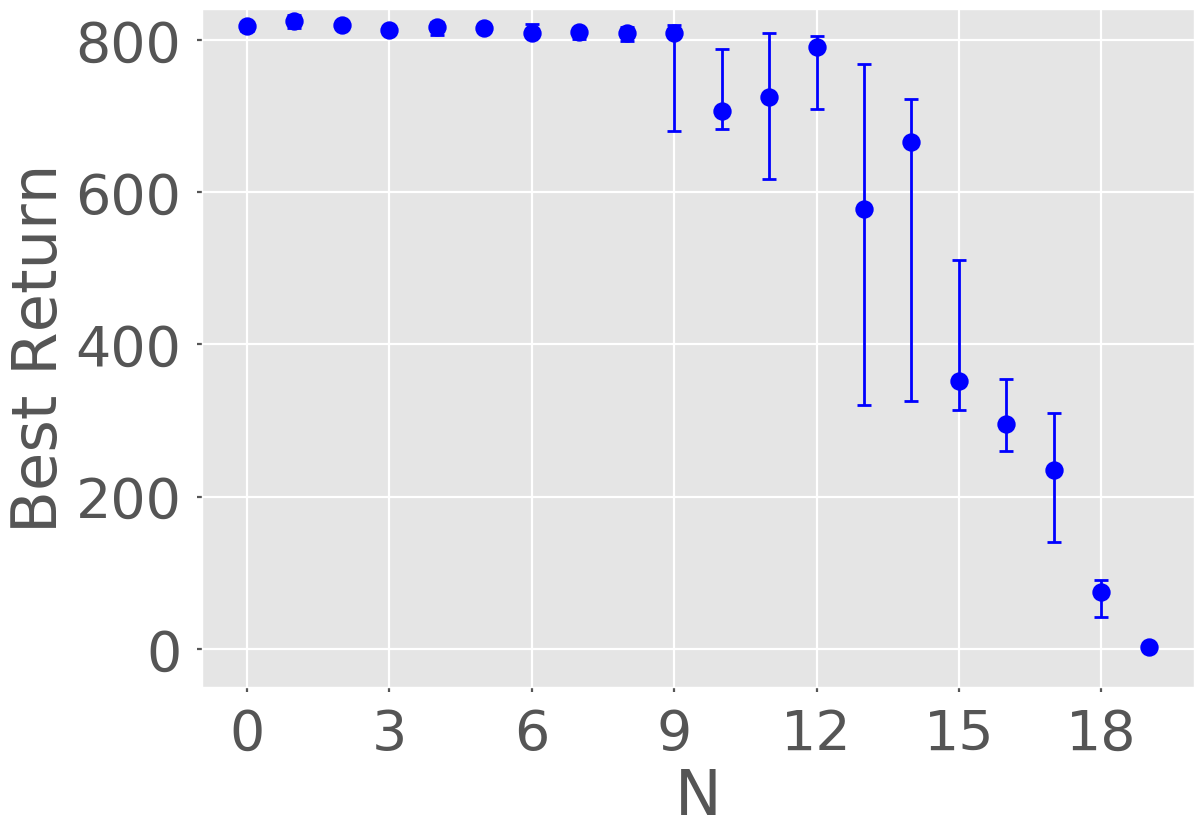}
		\subcaption{$\eta_1=0$}
	\end{subfigure}
	\begin{subfigure}{.32\textwidth}
		\includegraphics[width=\textwidth]{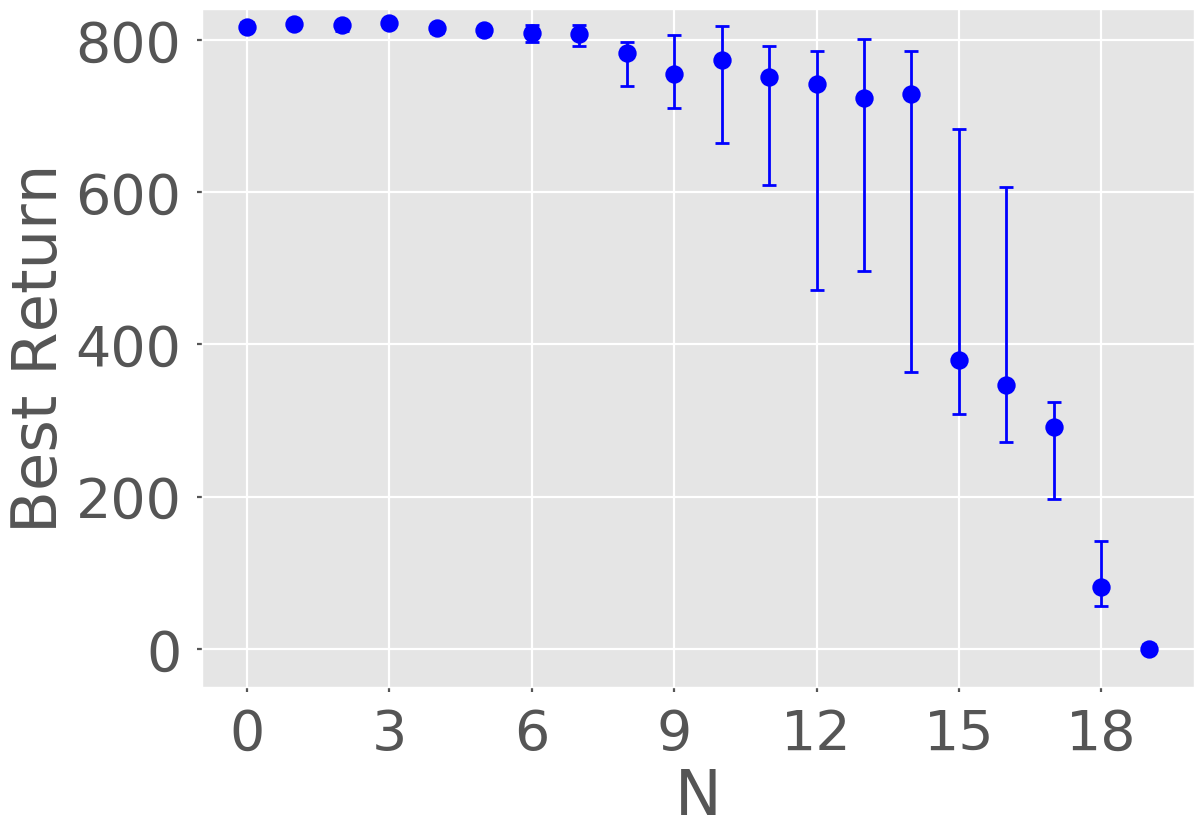}
		\subcaption{$\eta_1=0.1$}
	\end{subfigure}
	\begin{subfigure}{.32\textwidth}
		\includegraphics[width=\textwidth]{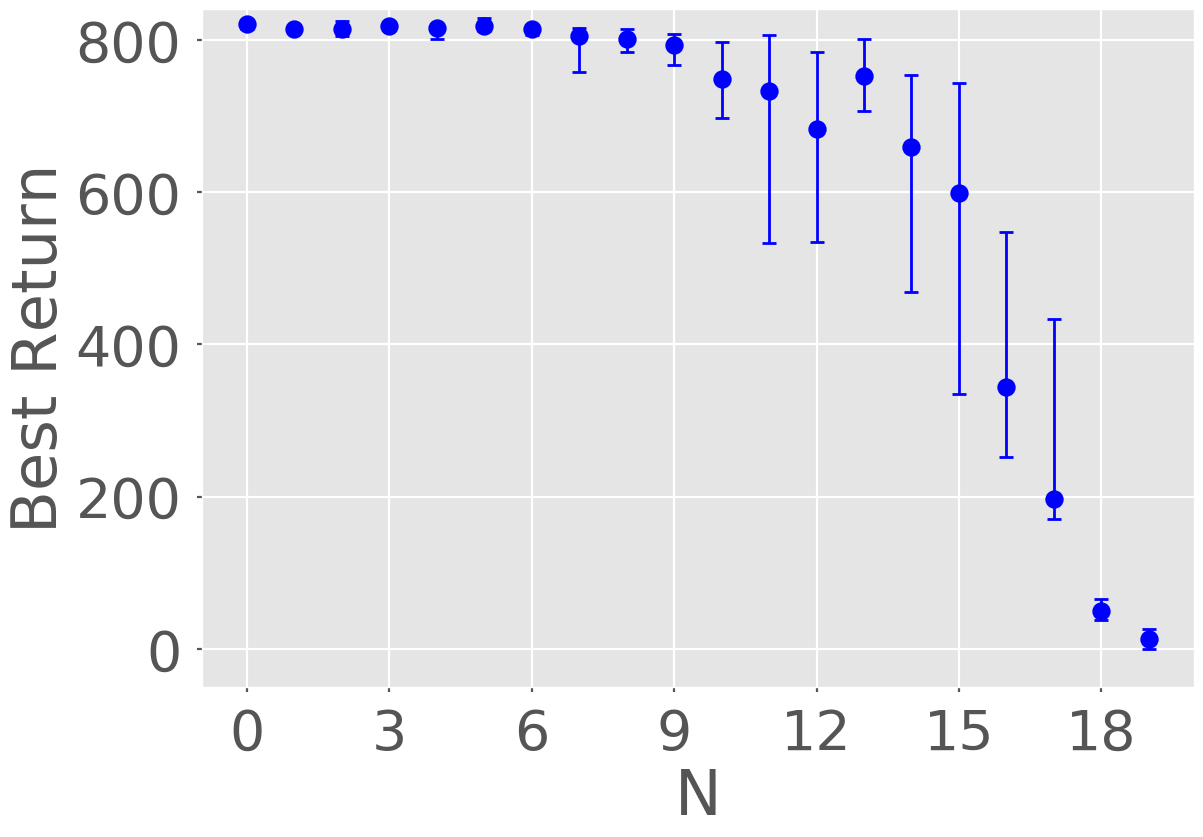}
		\subcaption{$\eta_1=0.2$}
	\end{subfigure}
	\begin{subfigure}{.32\textwidth}
		\includegraphics[width=\textwidth]{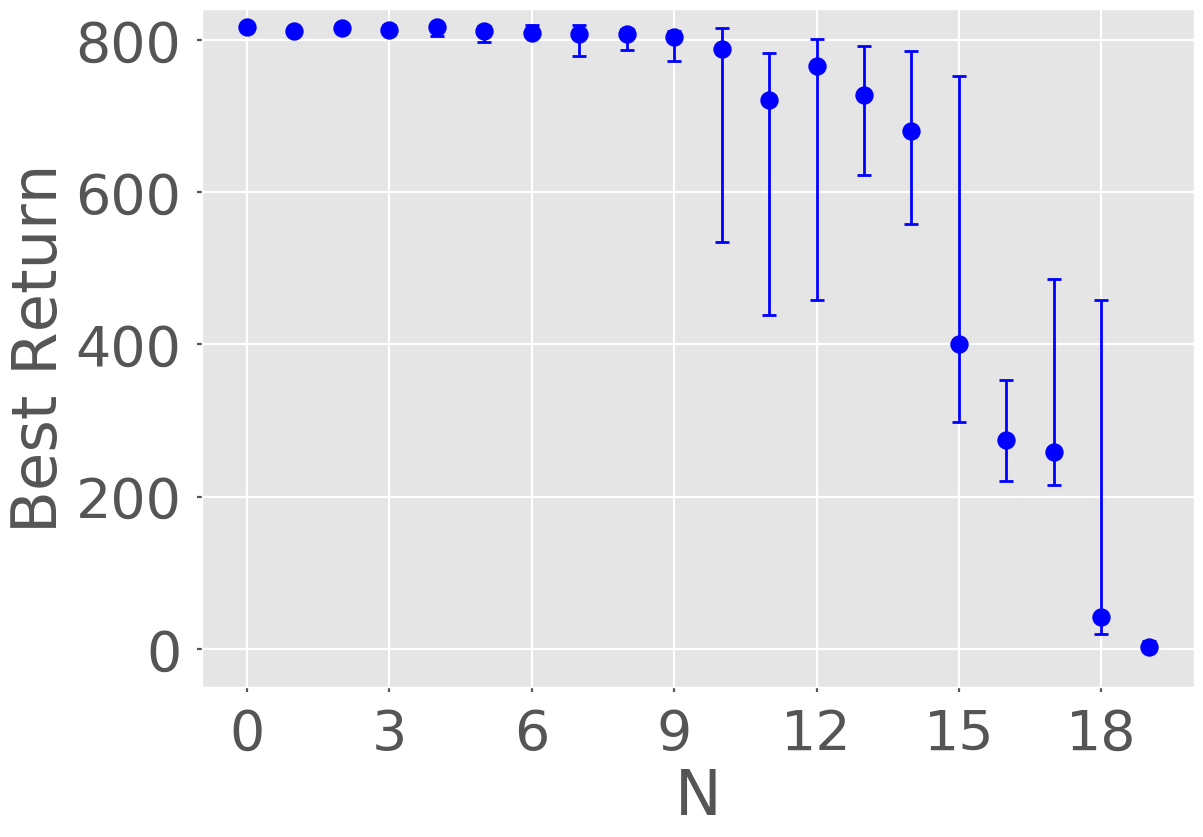}
		\subcaption{$\eta_1=0.3$}
	\end{subfigure}
	\begin{subfigure}{.32\textwidth}
		\includegraphics[width=\textwidth]{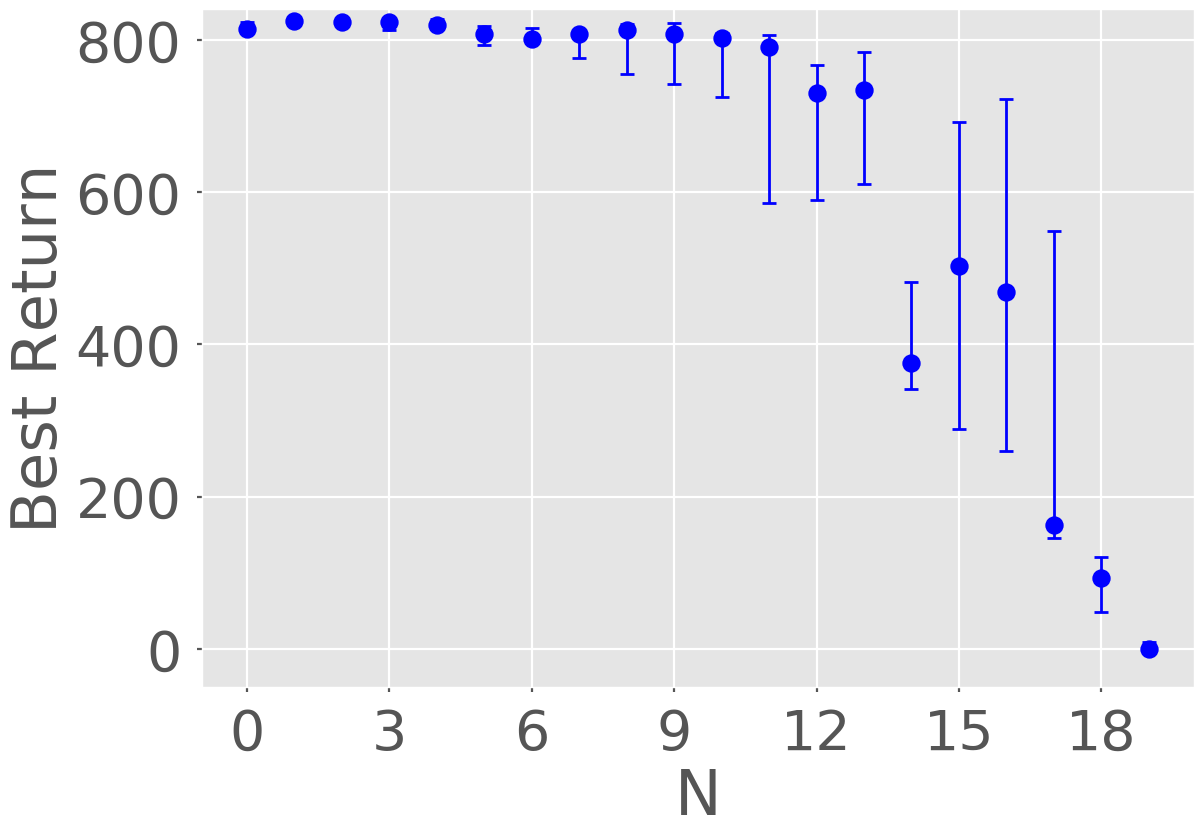}
		\subcaption{$\eta_1=0.4$}
	\end{subfigure}
	\begin{subfigure}{.32\textwidth}
		\includegraphics[width=\textwidth]{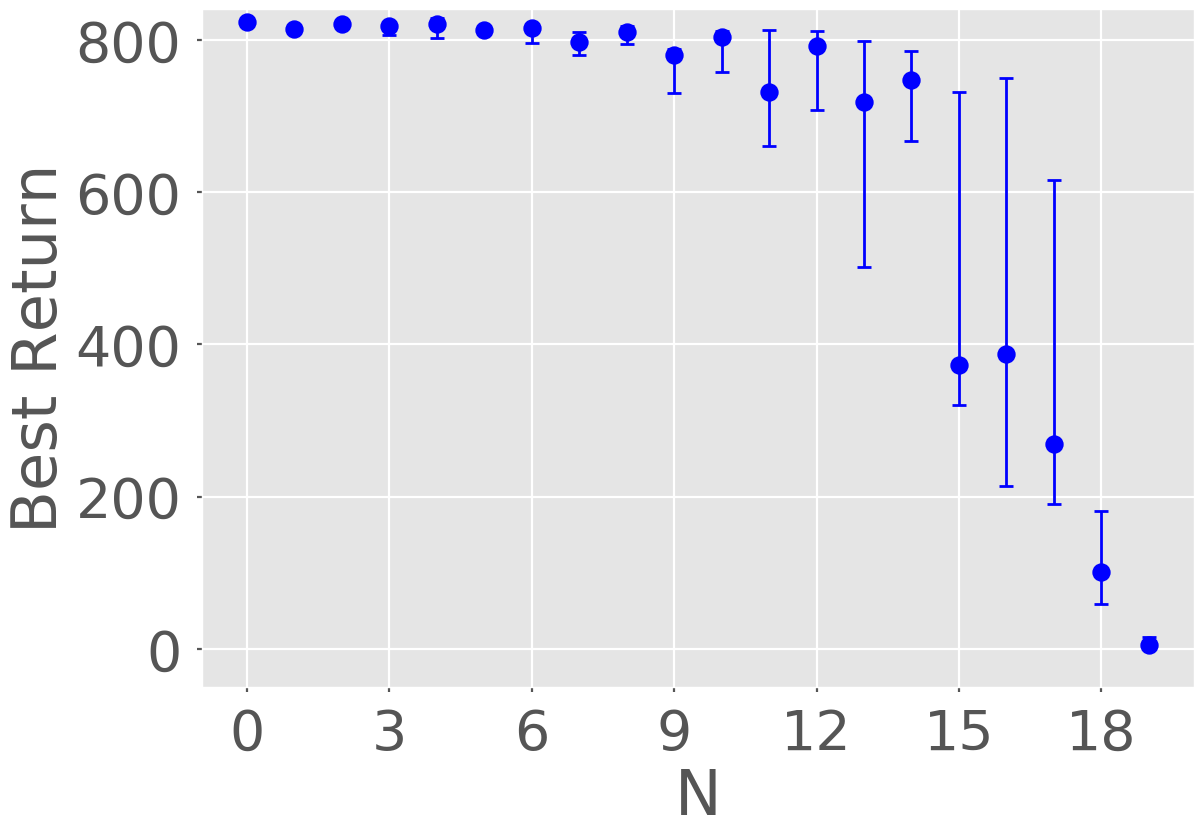}
		\subcaption{$\eta_1=0.5$}
	\end{subfigure}
	\begin{subfigure}{.32\textwidth}
		\includegraphics[width=\textwidth]{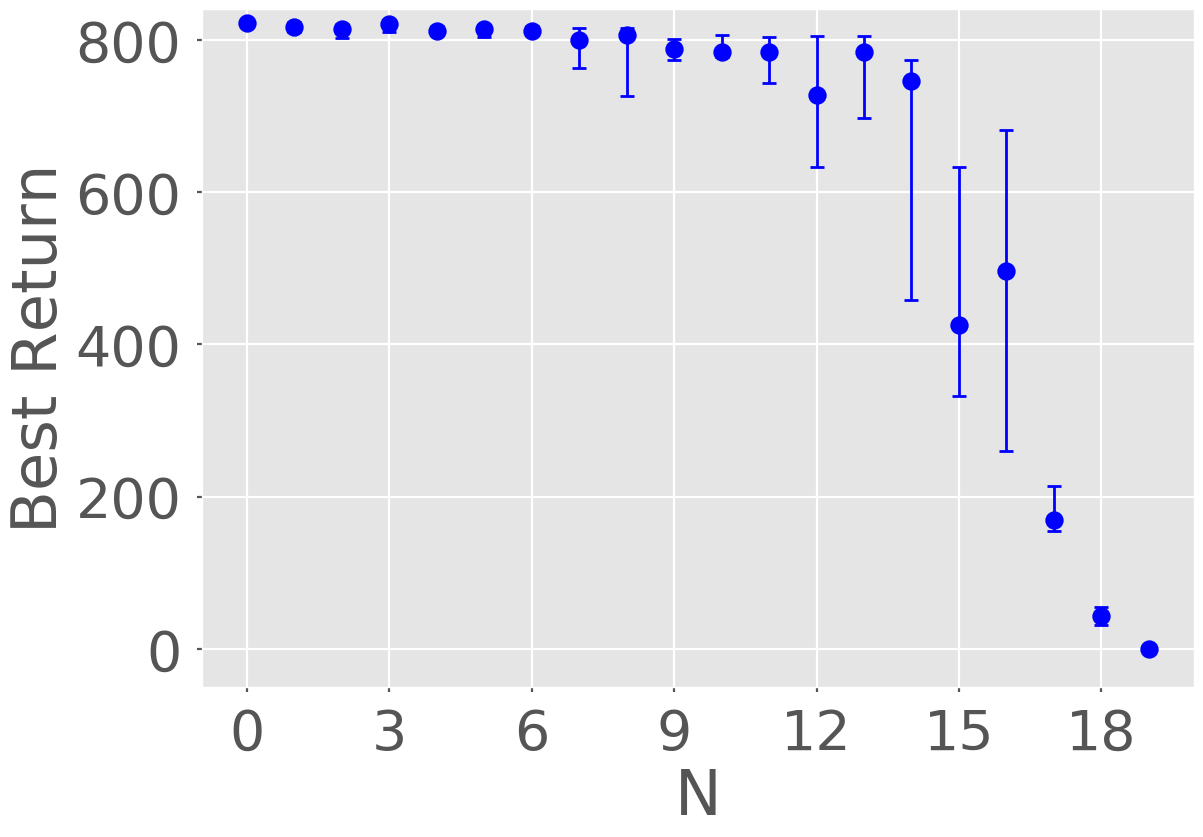}
		\subcaption{$\eta_1=0.6$}
	\end{subfigure}
	\begin{subfigure}{.32\textwidth}
		\includegraphics[width=\textwidth]{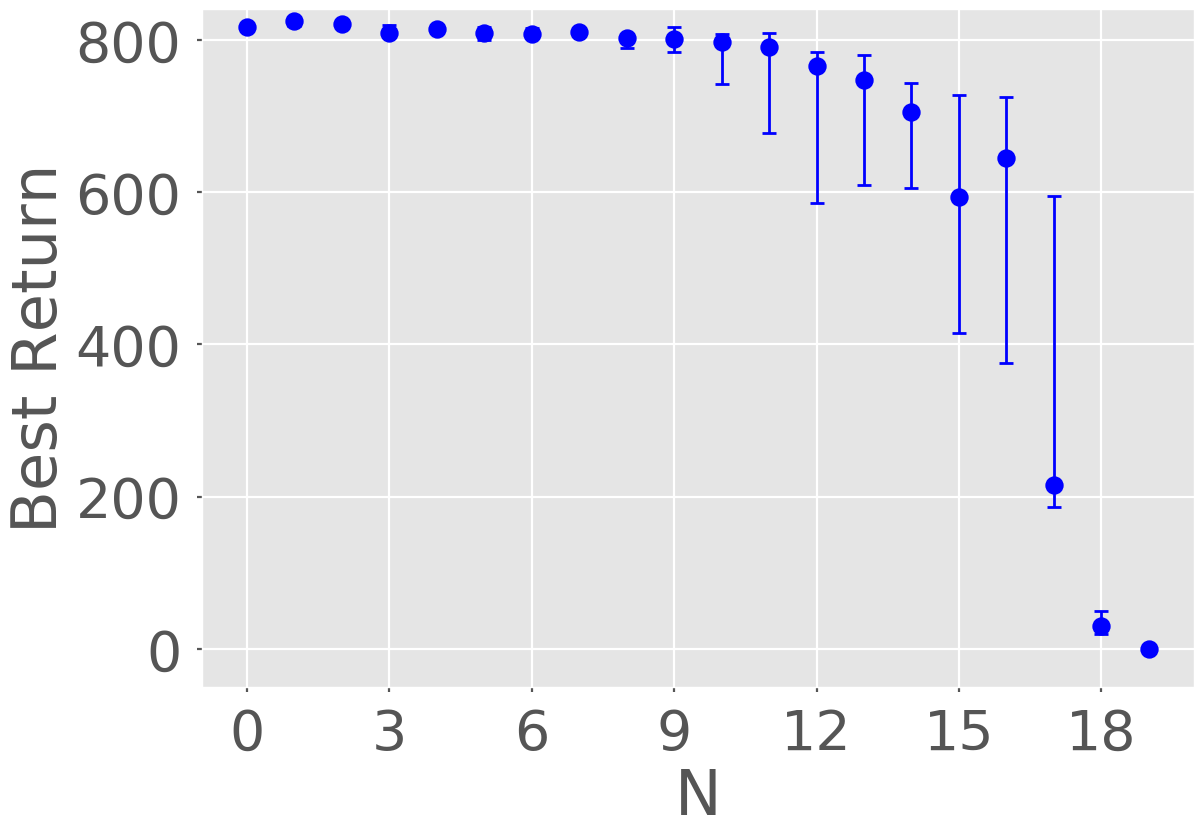}
		\subcaption{$\eta_1=0.7$}
	\end{subfigure}
	\begin{subfigure}{.32\textwidth}
		\includegraphics[width=\textwidth]{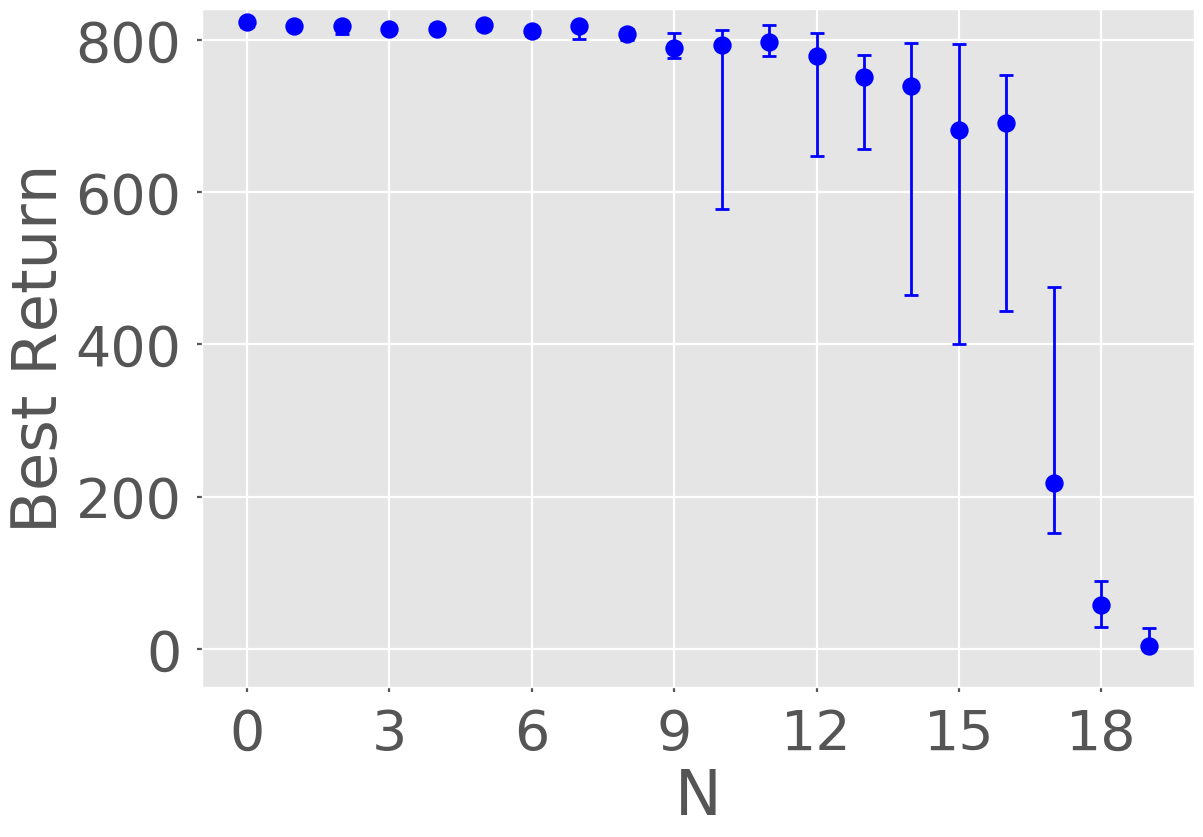}
		\subcaption{$\eta_1=0.8$}
	\end{subfigure}
	\begin{subfigure}{.32\textwidth}
		\includegraphics[width=\textwidth]{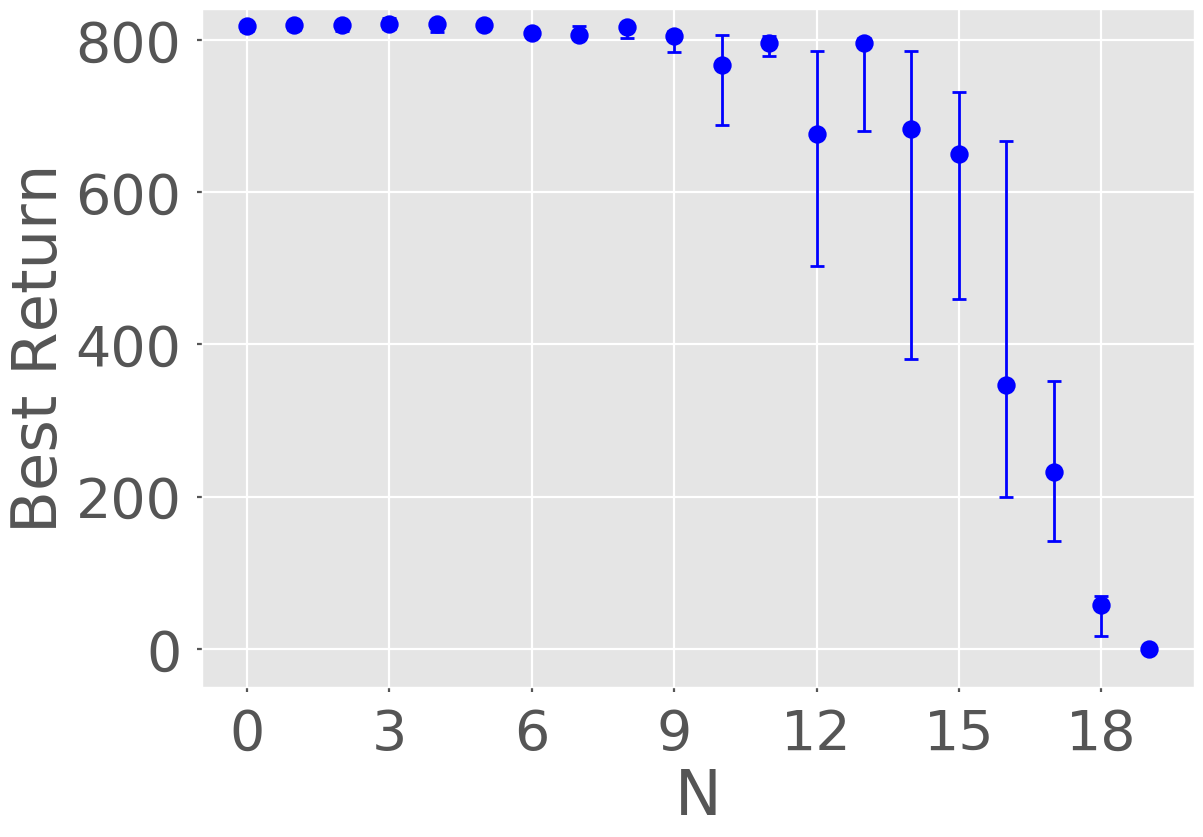}
		\subcaption{$\eta_1=0.9$}
	\end{subfigure}
	\begin{subfigure}{.32\textwidth}
		\includegraphics[width=\textwidth]{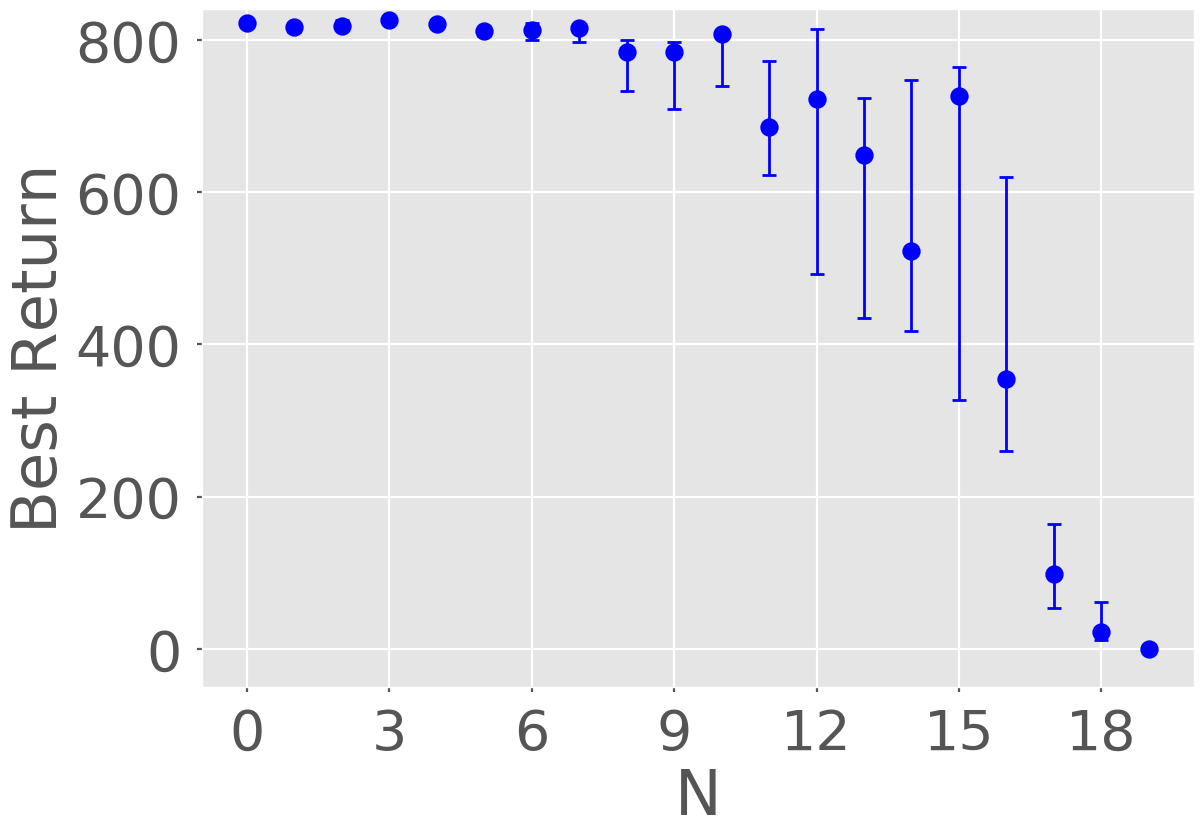}
		\subcaption{$\eta_1=1.0$}
	\end{subfigure}
	\caption{Ablation study for $\eta_1$.}
	\label{fig:ablation_eta1_all}
\end{figure}

\newpage
\subsubsection{Ablation study for $\eta_2$}
\begin{figure}[H]
	\centering
	\begin{subfigure}{.32\textwidth}
		\includegraphics[width=\textwidth]{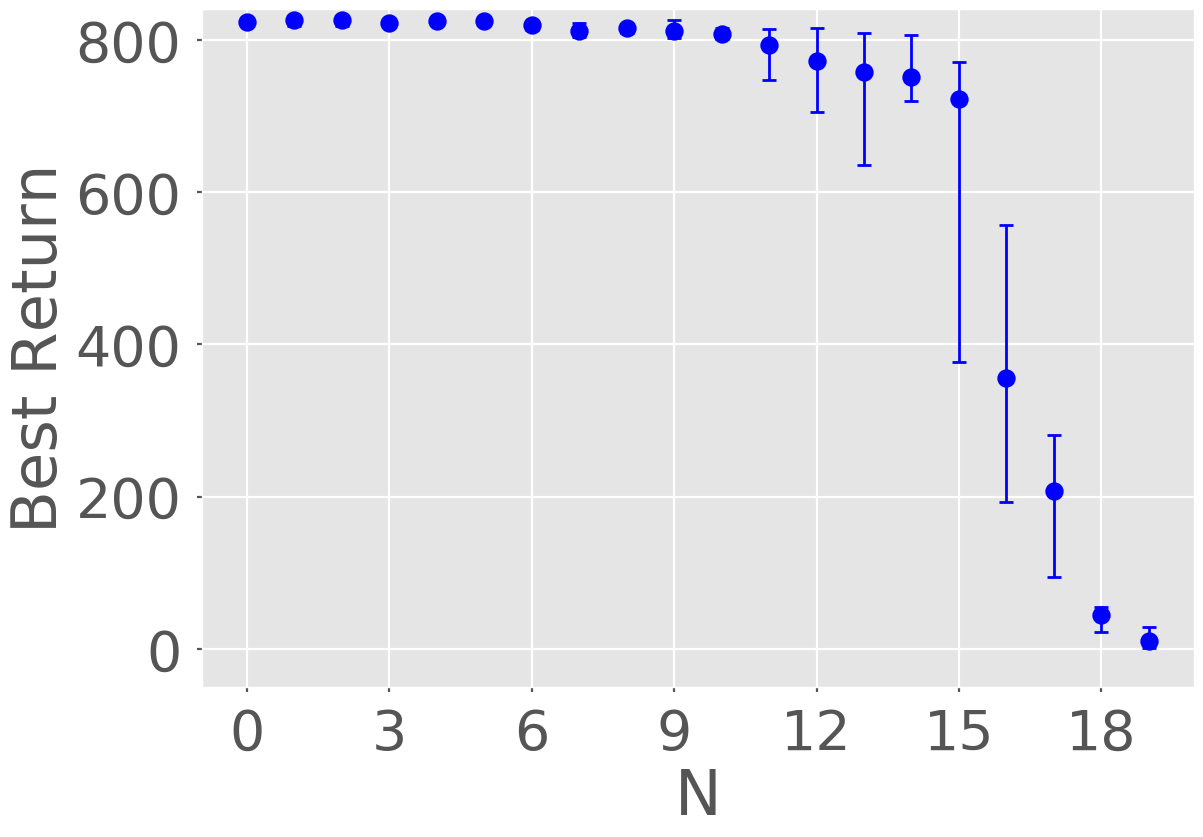}
		\subcaption{$\eta_2=0$}
		\label{fig:deep_sea_sto_N_0}
	\end{subfigure}
	\begin{subfigure}{.32\textwidth}
		\includegraphics[width=\textwidth]{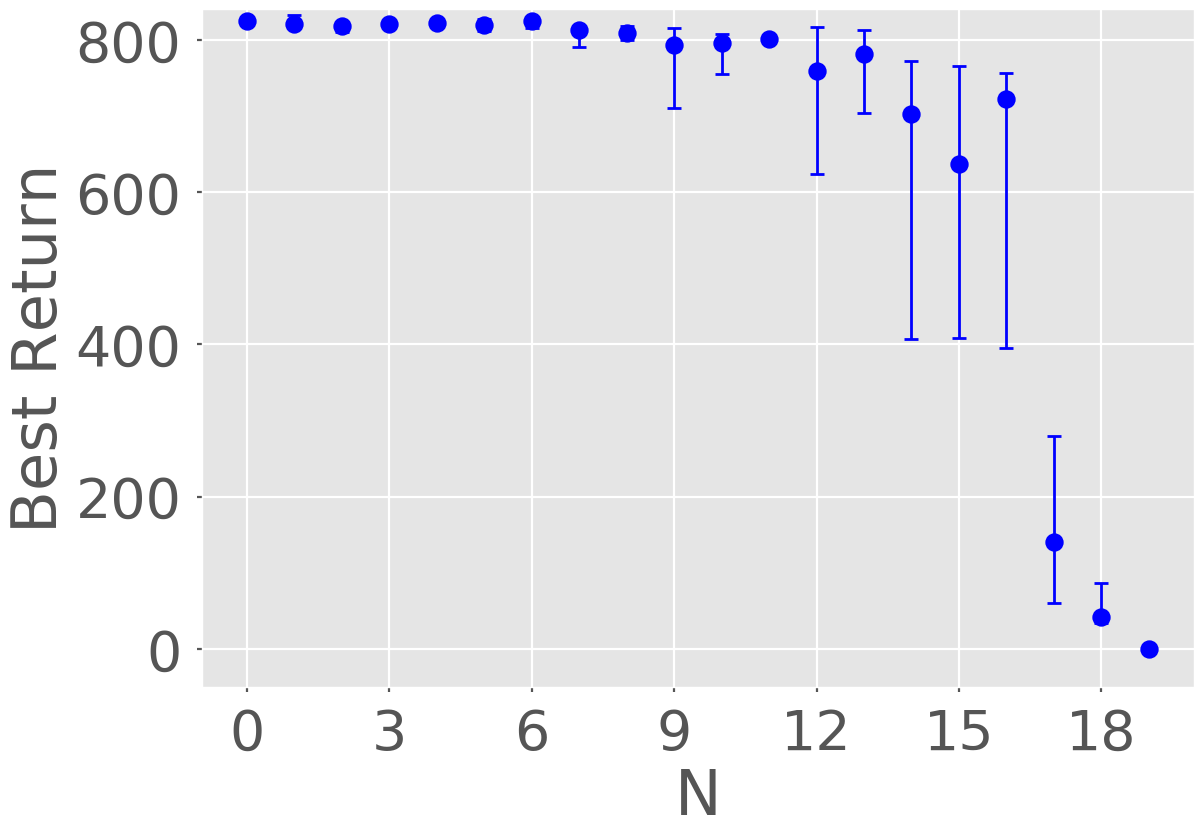}
		\subcaption{$\eta_2=0.1$}
		\label{fig:deep_sea_sto_N_1}
	\end{subfigure}
	\begin{subfigure}{.32\textwidth}
		\includegraphics[width=\textwidth]{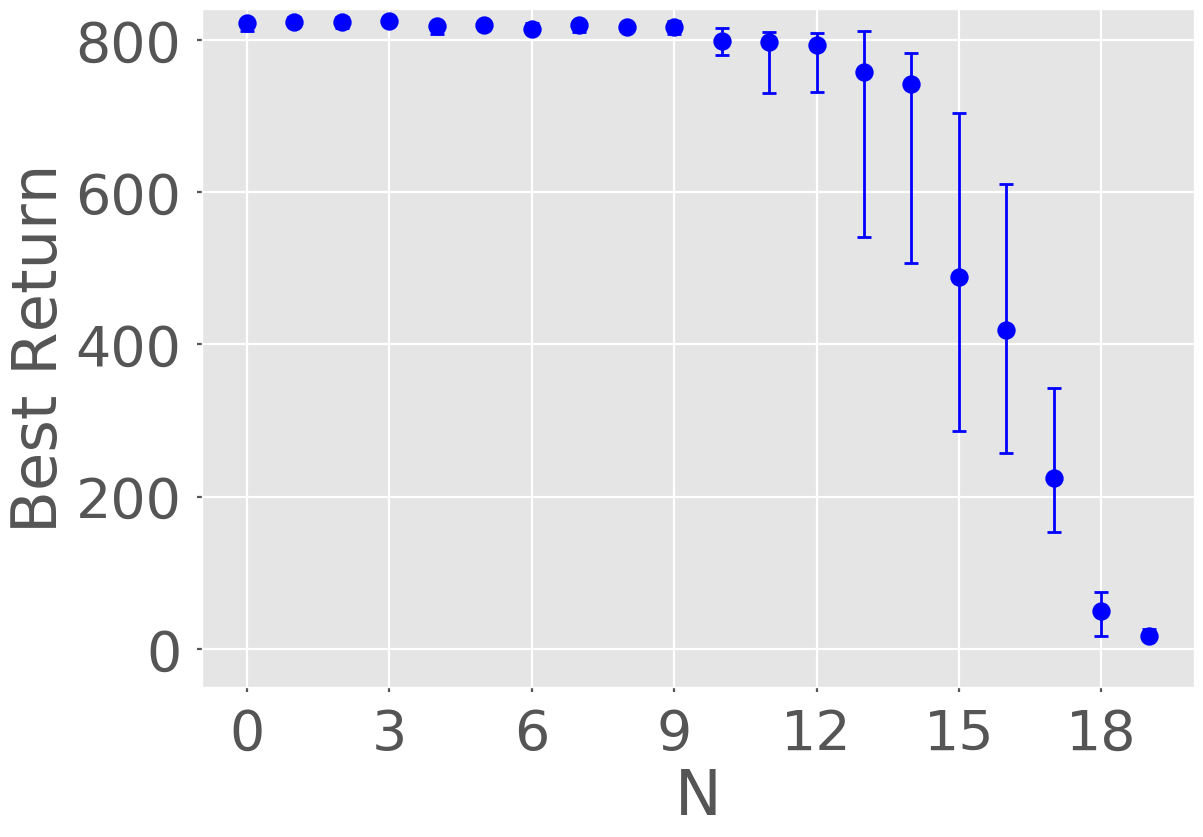}
		\subcaption{$\eta_2=0.2$}
		\label{fig:deep_sea_sto_N_2}
	\end{subfigure}
	\begin{subfigure}{.32\textwidth}
		\includegraphics[width=\textwidth]{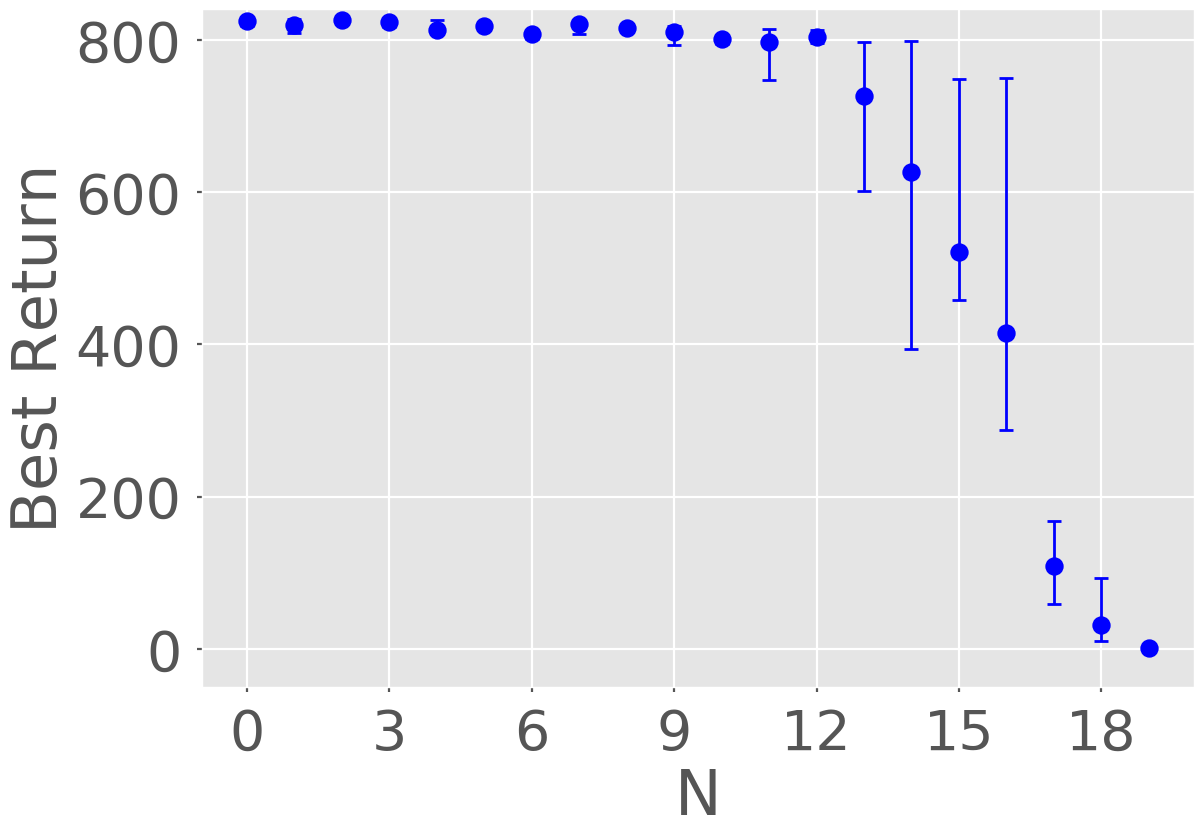}
		\subcaption{$\eta_2=0.3$}
		\label{fig:deep_sea_sto_N_3}
	\end{subfigure}
	\begin{subfigure}{.32\textwidth}
		\includegraphics[width=\textwidth]{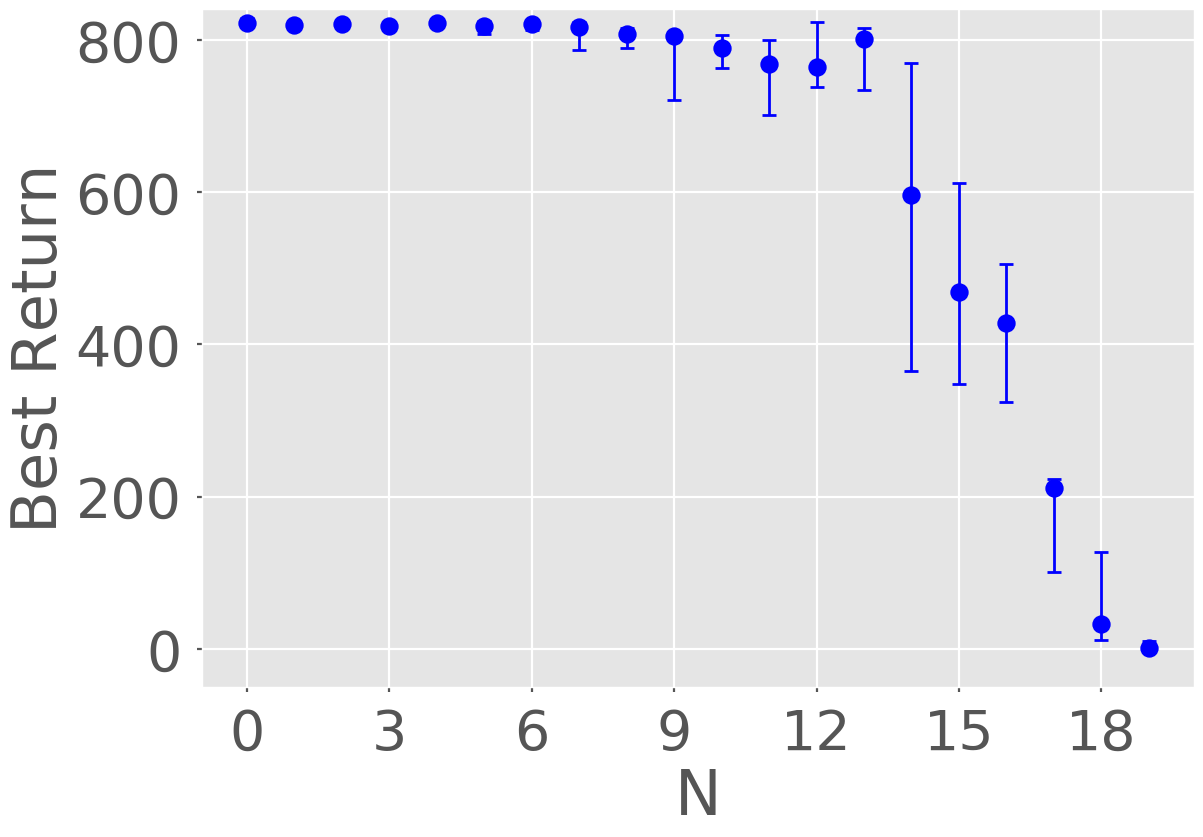}
		\subcaption{$\eta_2=0.4$}
		\label{fig:deep_sea_sto_N_4}
	\end{subfigure}
	\begin{subfigure}{.32\textwidth}
		\includegraphics[width=\textwidth]{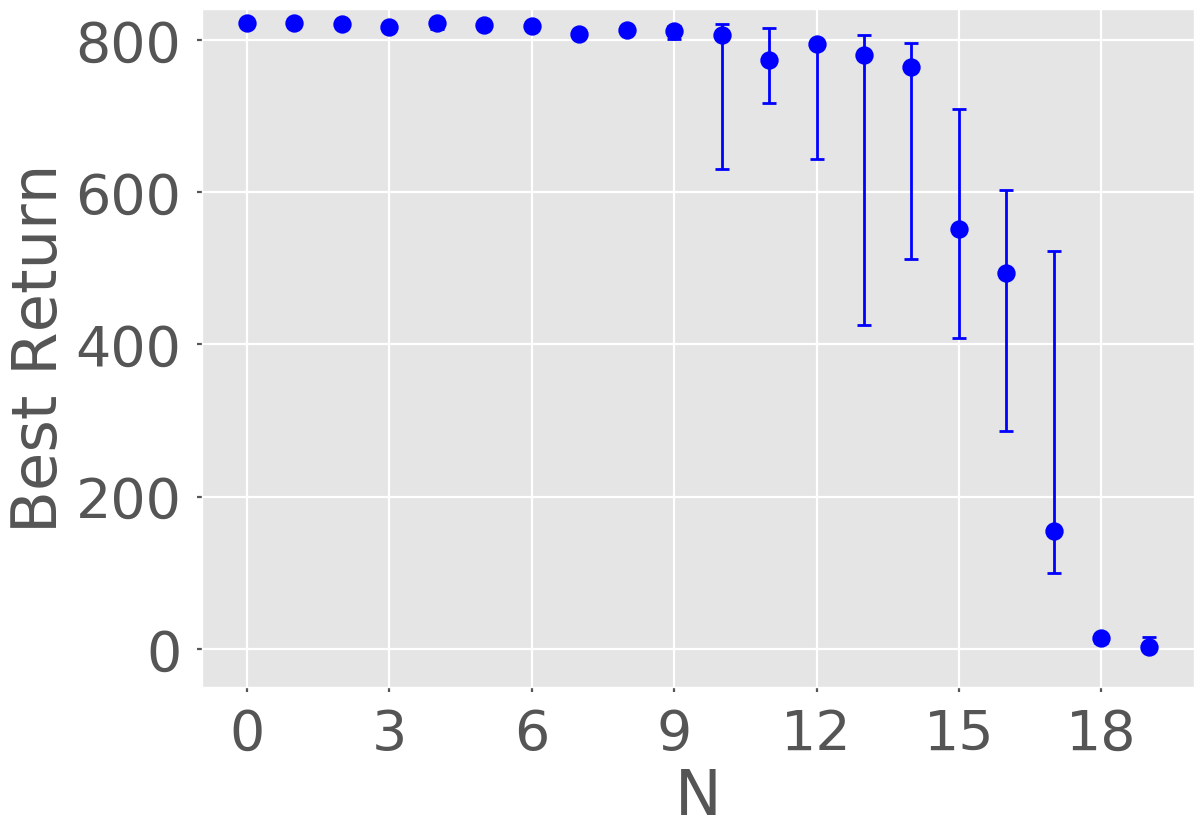}
		\subcaption{$\eta_2=0.5$}
		\label{fig:deep_sea_sto_N_5}
	\end{subfigure}
	\begin{subfigure}{.32\textwidth}
		\includegraphics[width=\textwidth]{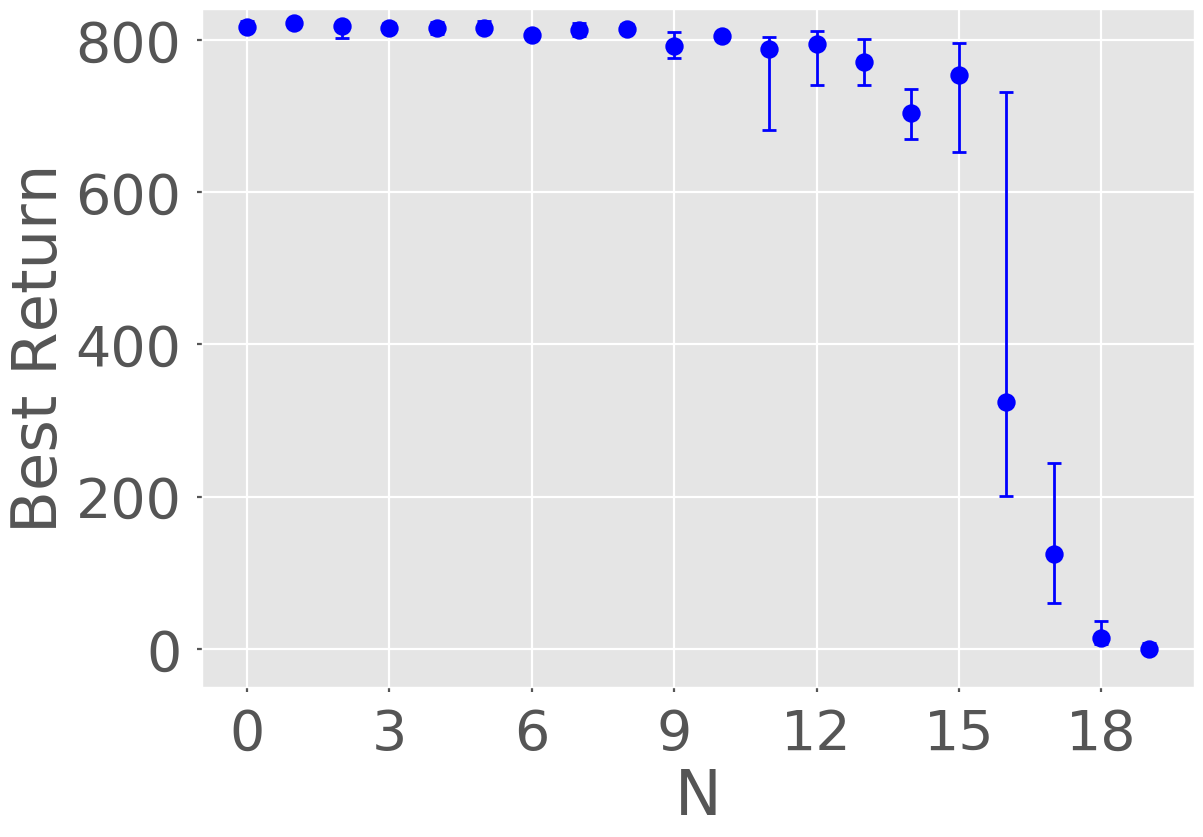}
		\subcaption{$\eta_2=0.6$}
		\label{fig:deep_sea_sto_N_6}
	\end{subfigure}
	\begin{subfigure}{.32\textwidth}
		\includegraphics[width=\textwidth]{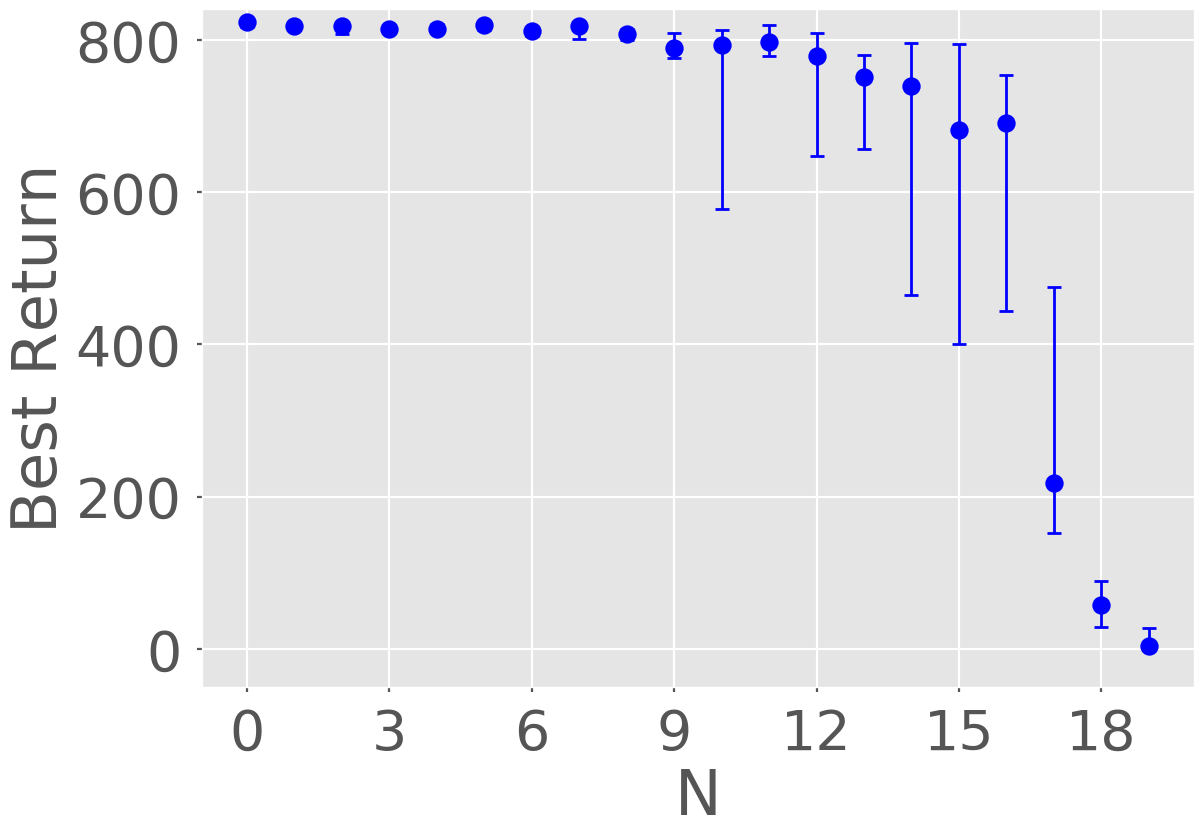}
		\subcaption{$\eta_2=0.7$}
		\label{fig:deep_sea_sto_N_7}
	\end{subfigure}
	\begin{subfigure}{.32\textwidth}
		\includegraphics[width=\textwidth]{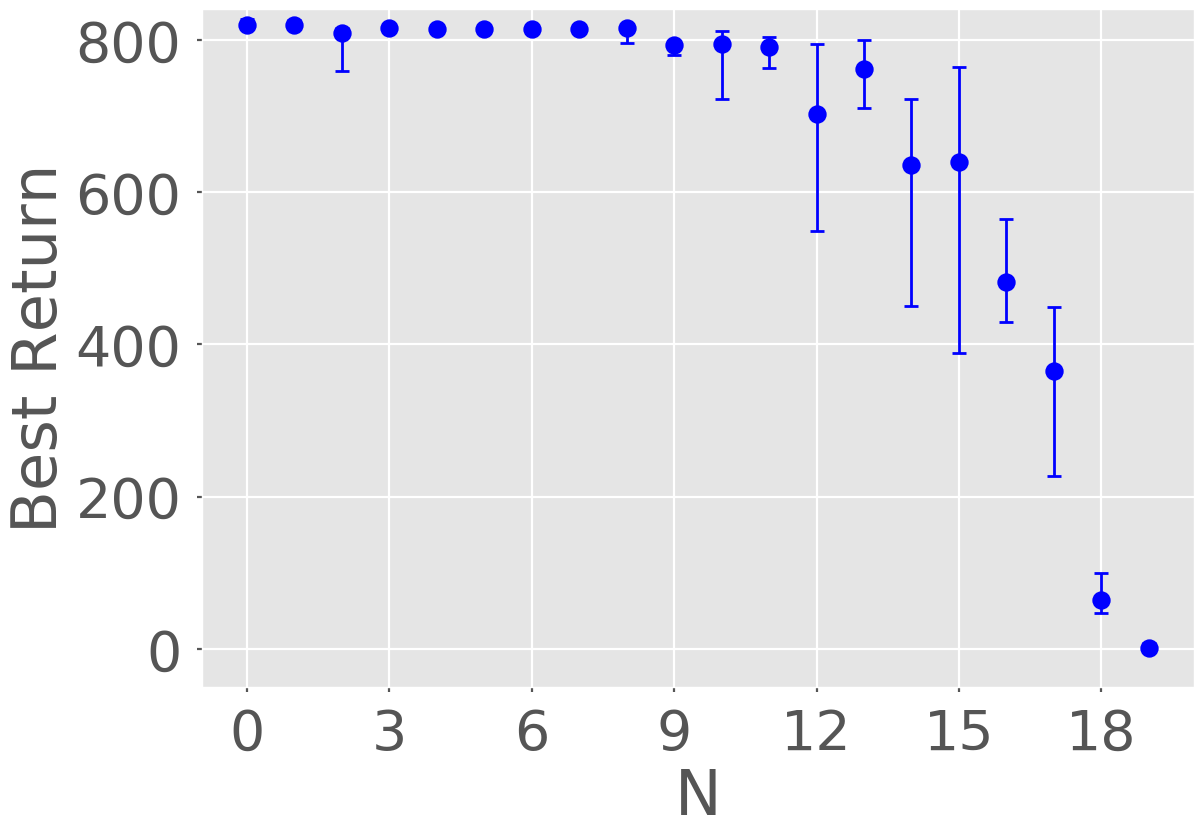}
		\subcaption{$\eta_2=0.8$}
		\label{fig:deep_sea_sto_N_8}
	\end{subfigure}
	\begin{subfigure}{.32\textwidth}
		\includegraphics[width=\textwidth]{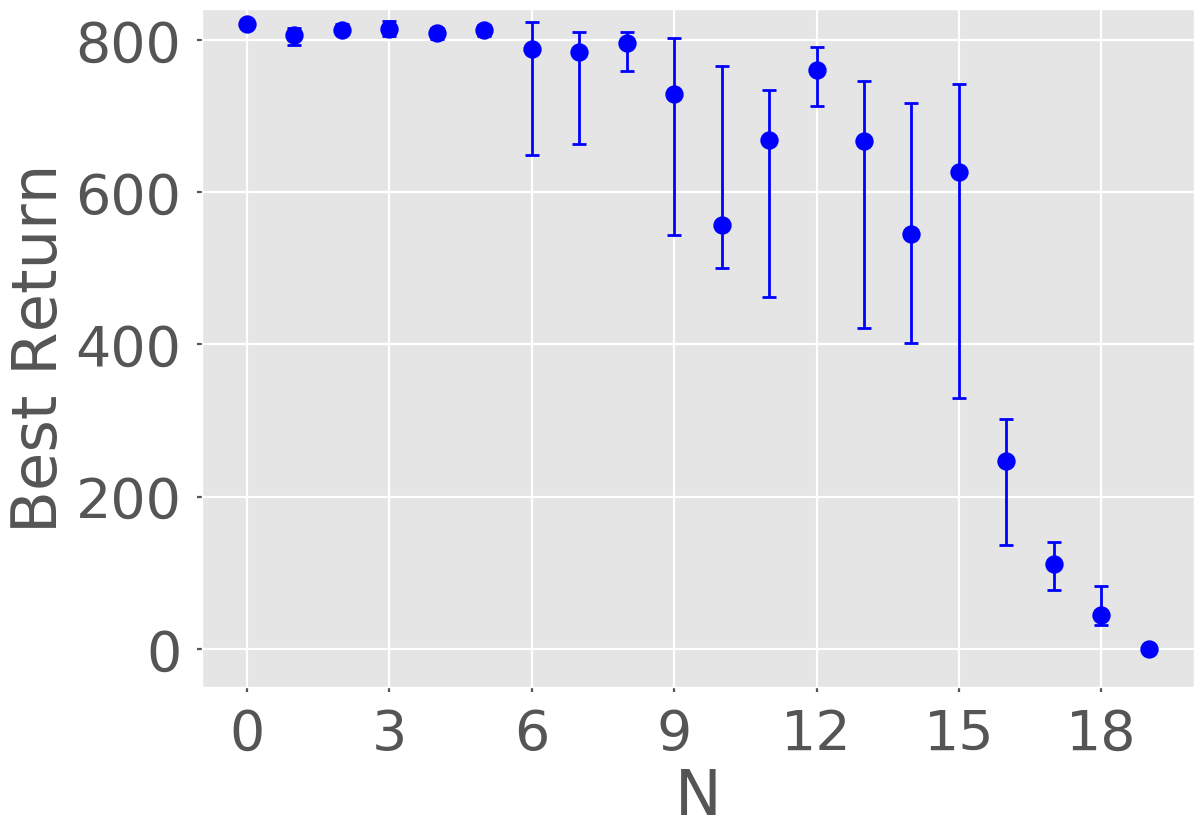}
		\subcaption{$\eta_2=0.9$}
		\label{fig:deep_sea_sto_N_9}
	\end{subfigure}
	\begin{subfigure}{.32\textwidth}
		\includegraphics[width=\textwidth]{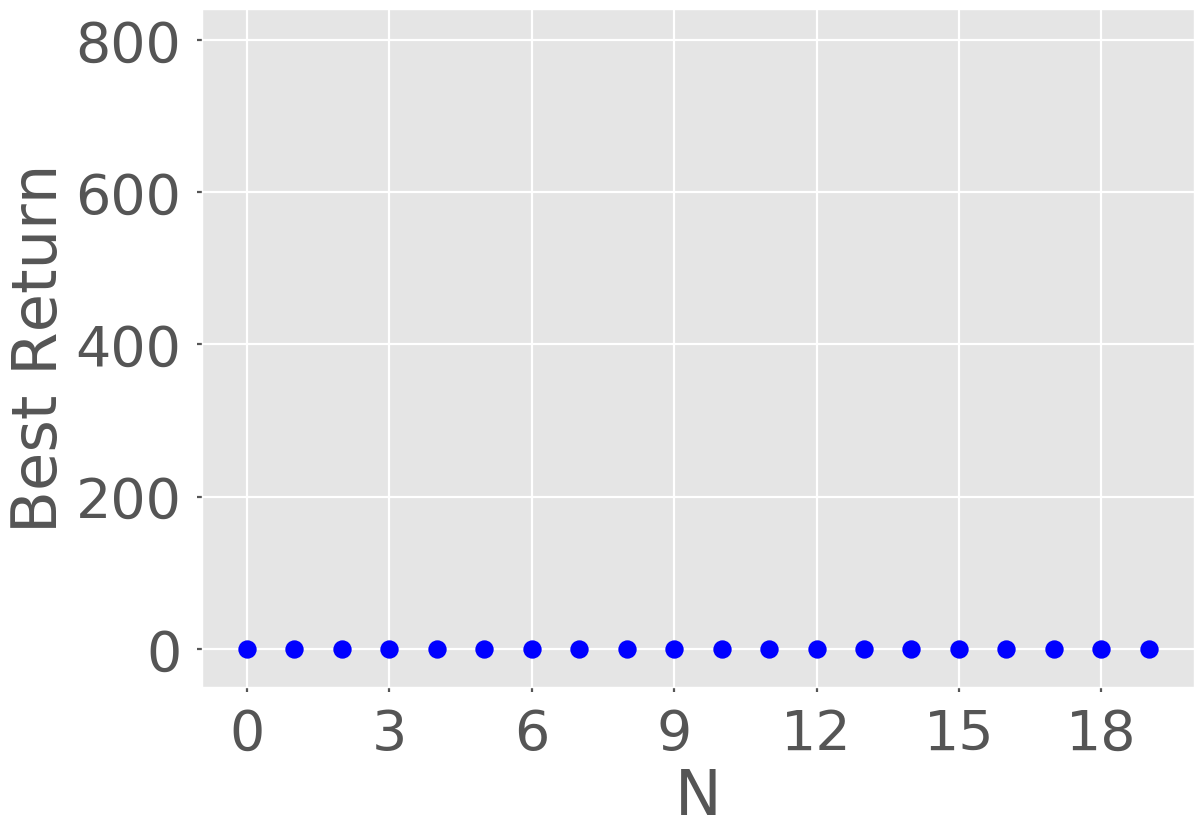}
		\subcaption{$\eta_2=1.0$}
		\label{fig:deep_sea_sto_N_10}
	\end{subfigure}
	\caption{Ablation study for $\eta_2$.}
	\label{fig:ablation_eta2_all}
\end{figure}
We clarify that figure \ref{fig:deep_sea_sto_N_10} shows no progress because the algorithm diverged for the chosen step-sizes in this ablation study.

\end{document}